%% file: output.tex
\documentclass{article}

\usepackage[preprint]{neurips_2025} 
\input{math_commands}

\usepackage[most]{tcolorbox}
\usepackage[colorlinks,citecolor=blue,linkcolor=blue,urlcolor=blue]{hyperref}
\tcbuselibrary{theorems}

\usepackage[capitalise]{cleveref}

\usepackage{sidecap}
\sidecaptionvpos{figure}{c}

\usepackage[utf8]{inputenc} 
\usepackage[T1]{fontenc}    
\usepackage{url}            
\usepackage{booktabs}       
\usepackage{amsfonts}       
\usepackage{nicefrac}       
\usepackage{microtype}      
\usepackage{xcolor}         
\usepackage{amsthm}
\usepackage{float}
\newtheorem{prop}{Proposition}
\newtheorem{corollary}{Corollary}
\newcounter{result}
\newcommand{\result}{R\arabic{result}}
\newcounter{mainsec}

\newtheorem{lemma}{Lemma}[mainsec]

\newcommand{\GAUSS}{\texttt{GAUSS}}
\usepackage{parskip}
\usepackage{caption}

\def\btheta{\bm{\theta}}
\def\bx{\bm{x}}

\title{Error analysis of a compositional score-based algorithm for simulation-based inference}

\author{%
  Camille Touron\\
  Univ. Grenoble Alpes, Inria
 \\
CNRS, Grenoble INP, LJK, France \\
  \texttt{camille.touron@inria.fr} \\
   \And
   Gabriel V. Cardoso \\
    Geostatistics team, Centre for geosciences and geoengineering\\
Mines Paris, PSL University, Fontaineableau, France\\
\texttt{gabriel.victorino\_cardoso@minesparis.psl.eu} \\
   \AND
   Julyan Arbel \\
 Univ. Grenoble Alpes, Inria
 \\
CNRS, Grenoble INP, LJK, France \\
   \texttt{julyan.arbel@inria.fr} \\
   \And
   Pedro L. C. Rodrigues \\
   Univ. Grenoble Alpes, Inria
 \\
CNRS, Grenoble INP, LJK, France \\
   \texttt{pedro.rodrigues@inria.fr} \\
}

\begin{document}
\allowdisplaybreaks[2]

\maketitle

\begin{abstract}

Simulation-based inference (SBI) has become a widely used framework in applied sciences for estimating the parameters of stochastic models that best explain experimental observations. A central question in this setting is how to effectively combine multiple observations in order to improve parameter inference and obtain sharper posterior distributions. Recent advances in score-based diffusion methods address this problem by constructing a compositional score, obtained by aggregating individual posterior scores within the diffusion process. While it is natural to suspect that the accumulation of individual errors may significantly degrade sampling quality as the number of observations grows, this important theoretical issue has so far remained unexplored. In this paper, we study the compositional score produced by the \texttt{GAUSS} algorithm of \citet{linhart2024diffusion} and establish an upper bound on its mean squared error in terms of both the individual score errors and the number of observations. We illustrate our theoretical findings on a Gaussian example, where all analytical expressions can be derived in a closed form.

\end{abstract}

\section{Introduction}
\tcbset{colback=white,colframe=black, boxrule=0.5pt, coltitle=white, fonttitle=\bfseries}
Probabilistic approaches to inverse problems~\citep{tarantola2005inverse} aim at inferring parameters $\btheta$ of a stochastic model from its outputs $\bx$ through the posterior distribution $p(\btheta|\bx)$. Directly sampling from this posterior is often difficult in practice, since the associated likelihood $p(\bx|\btheta)$ is frequently intractable. To address this, simulation-based inference with conditional score-based modeling~\citep{sharrocksequentialneuralscoreestimatio} approximates the posterior by learning a time-dependent conditional estimate $s_{\phi}(\btheta,\bx,t)$ of the score of noisy versions of the posterior, $\nabla_{\btheta}\log p_t(\btheta|\bx)$. Importantly, this training relies only on joint samples $(\btheta_i,\bx_i)\sim p(\btheta,\bx)$ which can be obtained sequentially as $\btheta_i\sim \lambda(\btheta)$, the prior, and then $\bx_i\vert \btheta_i\sim p(\bx\vert \btheta_i)$, thus circumventing the need to evaluate the likelihood (see Appendix~\ref{Appendix_A1} for details). A central question in this line of work is how the denoising score matching mean squared error, denoted $\epsilon_{\mathrm{DSM}}^2$, affects the quality of samples obtained via a backward diffusion process (see Figure~\ref{fig:sidecap}). Recent results by \citet{gao2025wasserstein} provide explicit bounds on the Wasserstein error of the approximate samples in terms of $\epsilon_{\mathrm{DSM}}^2$ and the discretization hyperparameters (number of steps, step size) used in the backward process. These results imply that one can reach a prescribed sampling accuracy (in Wasserstein distance) either by tuning the diffusion hyperparameters or by improving the accuracy of the score estimator.

We extend this setting to $n$ IID observations, with the goal of inferring parameters from the posterior $p(\btheta|\bx_{1:n})$. In practice, additional observations improve inference, as the posterior concentrates with growing $n$ (see Figure~\ref{fig:sidecap}). Recent works~\citep{linhart2024diffusion, geffner, arruda2025compositionalamortizedinferencelargescale, timeseriesgloecker} propose a compositional score approach: instead of directly approximating $\nabla_{\btheta}\log p_t(\btheta| \bx_{1:n})$ with a highly complex network, they aggregate individual posterior scores $s_\phi(\btheta, \bx_j, t) \approx \nabla_{\btheta}\log p_t(\btheta|\bx_j)$. The resulting compositional score can then be used in backward diffusion or annealed Langevin dynamics to sample from the multi-observation posterior. {Importantly, if we can bound the error of the compositional score estimate, then we can guarantee convergence of our sampling to the target multi-observation posterior in terms of Wasserstein distance \citep{gao2025wasserstein}.}
 
However, while the error of each individual score estimate can be controlled during training time, the error of the aggregated compositional score arises from the accumulation of individual errors as $n$ grows. To the best of our knowledge, this accumulation effect has not been theoretically analyzed so far. Here we address this gap: focusing on the compositional score~(\ref{eq:compositional-score}) defined by the \texttt{GAUSS} algorithm of \citet{linhart2024diffusion}, we derive in Proposition~\ref{prop:diffusion_discretization_2} an upper bound on the mean squared error (MSE) of its estimate~(\ref{eq:score-estimate}) as a function of the $n$ individual score errors.

\begin{SCfigure}[1.0][t]
    \centering
    \includegraphics[width=0.7\linewidth]{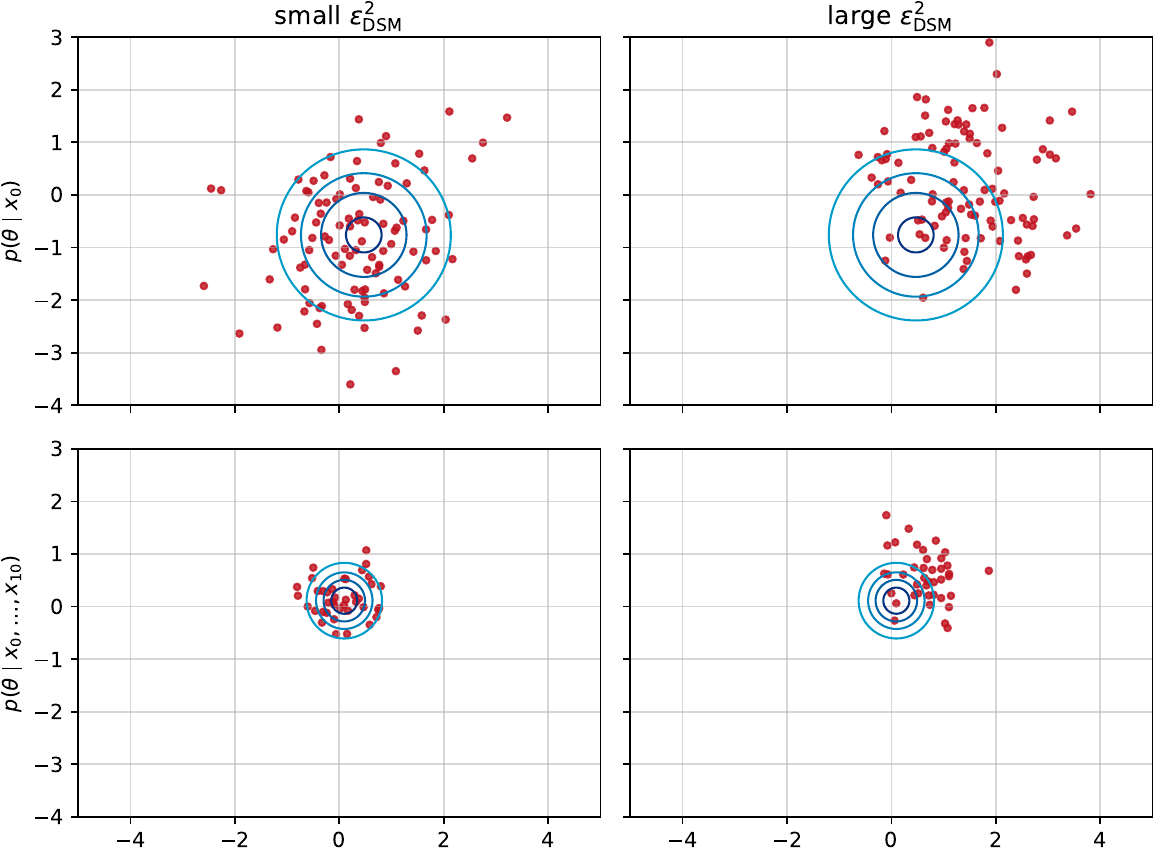}
    \caption{\footnotesize Contours represent the true target posterior distribution $p$ (see Appendix~\ref{gaussian_ex} for details), either conditioned on a single observation (top row) or eleven of them (bottom row): increasing the number of conditional observations sharpens the posterior and allows for better inference. 
    Scattered points in red are samples from $\tilde p$ obtained by a diffusion process with an inexact score estimate with mean squared error $\epsilon_\mathrm{DSM}^2$: larger errors tend to bias the sampling and degrade its quality.
    }
    \label{fig:sidecap}
\end{SCfigure}

\section{Background on compositional score estimation with \texttt{GAUSS}} 
\citet{linhart2024diffusion} adopt the common assumption~\citep{sohl2015deep} that the backward kernels of the diffused prior and the diffused individual posteriors can be well approximated by Gaussian distributions for all times $t\in[0,T]$ of a diffusion process, namely
\begin{equation*}
    \hat q^{\lambda}_{0|t}(\btheta_0|\btheta_t)=\mathcal{N}\big(\btheta_0;\mu_{t,\lambda}(\btheta_t), \Sigma_{t,\lambda}(\btheta)\big)~~\text{and}~~\hat q_{0|t}(\btheta_0| \btheta_t,\bx_j)=\mathcal{N}\big(\btheta_0;\mu_{t,j}(\btheta_t), \Sigma_{t,j}(\btheta)\big),
\end{equation*}
with $j \in \{1, \dots, n\}$. In addition, their \GAUSS~algorithm assumes that both the prior and each individual posterior are themselves Gaussian, $\mathcal{N}(\btheta;\mu_{\lambda},\Sigma_{\lambda})$ and $\mathcal{N}(\btheta;\mu_j,\Sigma_j)$. Under this assumption, the diffused covariances $\Sigma_{t,\lambda}$ and $\Sigma_{t,j}$ are fully determined and, importantly, no longer depend on $\btheta$. With these simplifications, the diffused compositional score takes the form
\begin{equation}
    \label{eq:compositional-score}
    \nabla_{\btheta} \log p_t(\btheta \vert \bx_{1:n})=\Lambda^{-1}\Bigg(\sum_{j=1}^n \Sigma_{t,j}^{-1}\nabla_{\btheta}\log p_t(\btheta|\bx_j)-(n-1)\Sigma_{t,\lambda}^{-1}\nabla_{\btheta}\log \lambda_t(\btheta)\Bigg),
\end{equation}
where $\Lambda = \sum_j \Sigma^{-1}_{t,j} - (n-1)\Sigma^{-1}_{t, \lambda}$. {Note that the compositional score~(\ref{eq:compositional-score}) does not stem from the simple addition of the $n$ individual scores, as it aims at correctly approximating the true score of a noisy version of the multi-observation posterior for some specific noise level prescribed by a classical diffusion process.} For most priors used in the SBI literature (Gaussian, Uniform, log-Normal), the diffused score $\nabla_{\btheta}\log \lambda_t(\btheta)$ has a closed-form expression and therefore introduces no error~\citep{sharrocksequentialneuralscoreestimatio}. In contrast, Equation~(\ref{eq:compositional-score}) also involves the scores of the individual posteriors, which are only available through the approximations $s_\phi(\btheta,\bx_j,t)$ and thus contribute to errors $\epsilon_{\mathrm{DSM},j}^2$. A further source of bias comes from the estimation of the covariance matrices $\Sigma_{\lambda}$ and $\Sigma_j$. Covariance $\Sigma_{\lambda}$ can be approximated by the empirical covariance of prior samples. However, for each posterior one must simulate a backward diffusion process using the approximate score $s_{\phi}(\btheta,\bx_j,t)$ to generate samples from $\tilde{p}(\btheta|\bx_j) \approx {p}(\btheta|\bx_j)$, and then compute the empirical covariance $\tilde \Sigma_j$. The estimator for the compositional score~(\ref{eq:compositional-score}) is then written as:
\begin{align}\label{eq:score-estimate}
    s(\btheta,\bx_{1:n},t)&= \tilde \Lambda^{-1}\left(\sum_{j=1}^n \tilde \Sigma_{t,j}^{-1}s_{\phi}(\btheta,\bx_j,t)-(n-1)\tilde \Sigma_{t,\lambda}^{-1}\nabla_{\btheta}\log \lambda_t(\btheta)\right).
\end{align}
In what follows, we derive an upper bound on the mean squared error of this estimator in terms of the individual score errors $\epsilon_{\mathrm{DSM},j}^2$, the errors incurred in estimating the precision matrices ${\Sigma}^{-1}_\lambda$ and ${\Sigma}^{-1}_j$, and the number of observations $n$. We use $\Vert.\Vert$ to refer either to the Euclidean norm when applied to vectors, or spectral norm when applied to matrices.

\section{Contributions}
We first derive an upper bound on the error of the estimator for the precision matrices, noting that for the \GAUSS~ algorithm the errors $\Vert\Sigma^{-1}_{t,j}-\tilde \Sigma^{-1}_{t,j}\Vert$ at each $t\in[0,T]$ are directly related to the initial estimation error $\Vert\Sigma^{-1}_j-\tilde \Sigma^{-1}_j\Vert$, which in turn depends on the sampling quality from each individual posterior. Since Gaussian distributions are smooth log-concave, we can use the aforementioned results from~\citet{gao2025wasserstein} to tune the discretization hyperparameters of each backward diffusion process and drive the individual score error $\epsilon_{\mathrm{DSM},j}^2$ sufficiently low to achieve a prescribed small Wasserstein error. The following proposition (proof in Appendix~\ref{proof_prop_1}) provides a bound on the precision estimation error as a function of a fixed Wasserstein error $\eta$.
\def\statementprecisionmatrix{
Choosing $0 < \eta < \min\Big(\sqrt{\frac{\|\Sigma\|}{2}},f(\|\Sigma\|, \|\Sigma^{-1}\|)\Big)$ we get
\begin{equation*}
      \mathcal{W}_2(p,\tilde p)\leq \eta \Rightarrow \Vert\tilde\Sigma-\Sigma\Vert\leq \gamma \Rightarrow \Vert\tilde\Sigma^{-1}-\Sigma^{-1}\Vert\leq \frac{\gamma\Vert\Sigma^{-1}\Vert^2}{1-\gamma\Vert\Sigma^{-1}\Vert},
\end{equation*}
where
\begin{equation*}
\gamma = \left(2\sqrt{2\|\Sigma\|}\eta +\eta^2(1+\sqrt{2})\right)\frac{\|\Sigma\|}{\|\Sigma\|-\eta\sqrt{2\|\Sigma\|}} > 0,
\end{equation*}
}
\begin{prop}[Precision matrix error] \label{prop:diffusion_discretization_1}
\statementprecisionmatrix
and function $f:\mathbb{R}_+^2\to\mathbb{R}_+$ is defined in Appendix~\ref{proof_prop_1}.
\end{prop}
We now propose an upper bound on the MSE between the compositional score~(\ref{eq:compositional-score}) and its estimate~(\ref{eq:score-estimate}) for any time $t\in [0,T]$ of the diffusion process as a function of the individual score errors, the precision estimation errors and the number of observations $n$. For simplicity, we assume a Gaussian prior, leading to closed-form formula for the corresponding scores and the precision matrices $\Sigma_{t,\lambda}^{-1}$ (a more general result is stated in Appendix~\ref{proof_prop_2}). We also assume that individual posterior score errors $\epsilon_{\mathrm{DSM},j}^2$(resp. precision errors) are bounded by the same constant $\epsilon_{\mathrm{DSM}}^2$ (resp. $\epsilon$). Note that in practice, the error $\epsilon$ can be directly linked to the Wasserstein error $\eta$ of each diffusion process, and thus indirectly to $\epsilon_{\mathrm{DSM}}^2$, as stated in Proposition~\ref{prop:diffusion_discretization_1}. 
Finally, all constants, except $\epsilon$, are time-dependent (proof in Appendix~\ref{intermediate_results}).

\begin{prop}[Compositional score error] \label{prop:diffusion_discretization_2}
Let  $\eta$ be chosen as in Proposition~\ref{prop:diffusion_discretization_1} and denote 
\begin{equation*}
\begin{array}{rcl}
M = \max_j \|\Sigma_{t, j}^{-1}\| &\text{with}& L=\max_j \mathbb{E}_{\btheta \sim p_t(\cdot \mid \bx_{1:n})}\left(\Vert\nabla\log p_t(\btheta| \bx_j)\Vert^2\right) \\[1em]
M_\lambda \geq \|\Sigma^{-1}_{t, \lambda}\| & \text{with} & L_{\lambda} \geq \mathbb{E}_{\btheta \sim p_t(\cdot  \mid \bx_{1:n})}\left(\Vert\nabla\log \lambda_{t}(\btheta)\Vert^2\right).
\end{array}
\end{equation*}
Suppose that  $\mathbb{E}_{\btheta \sim p_t(\cdot |\bx_{1:n})}(\Vert \nabla_{\btheta}\log p_t(\btheta|\bx_j)-s_{\phi}(\btheta,\bx_j,t)\Vert^2) \leq \epsilon_{\mathrm{DSM}}^2$ and $\Vert\Sigma_j^{-1}-\tilde \Sigma_j^{-1}\Vert\leq \epsilon$ for all $j=1,\ldots,n$ such that $n \epsilon < \frac{1}{\Vert\Lambda^{-1}\Vert}$. Then it holds 
\begin{multline*}
         \mathbb{E}_{\btheta \sim p_t(\cdot \mid \bx_{1:n})}\big[\Vert\nabla\log p_t(\btheta\mid \bx_{1:n})-s(\btheta,\bx_{1:n},t)\Vert^2 \big]
         \leq 
         \left[(n-1)\Vert\Lambda ^{-1}\Vert\sqrt{L_{\lambda}}\left(\frac{n\epsilon\Vert\Lambda^{-1} \Vert M_{\lambda}}{1-n\epsilon\Vert\Lambda^{-1} \Vert}\right)\right.\\
         \left.
         +n \Vert\Lambda ^{-1}\Vert(\sqrt{L}+\epsilon_{\mathrm{DSM}})\left(\epsilon+\frac{n\epsilon\Vert\Lambda^{-1} \Vert(M+\epsilon)}{1-n\epsilon\Vert\Lambda^{-1} \Vert} \right)
         +\Vert\Lambda ^{-1}\Vert n M\epsilon_{\mathrm{DSM}}\right]^2.
\end{multline*}
If $\epsilon_{\mathrm{DSM}}$ is sufficiently small and accompanied by a proper choice of diffusion hyperparameters such that $\mathcal{W}_2(\tilde p(\btheta\mid \bx_j),p(\btheta\mid \bx_j))\leq \eta$ for all $j=1,\ldots,n$ (Proposition 4 in~\citealp{gao2025wasserstein}), then the precision estimation error $\epsilon$ can be further bounded using Proposition~\ref{prop:diffusion_discretization_1} and $\eta$.
\end{prop}

\begin{figure}[h]
    \centering
    \includegraphics[width=\linewidth]{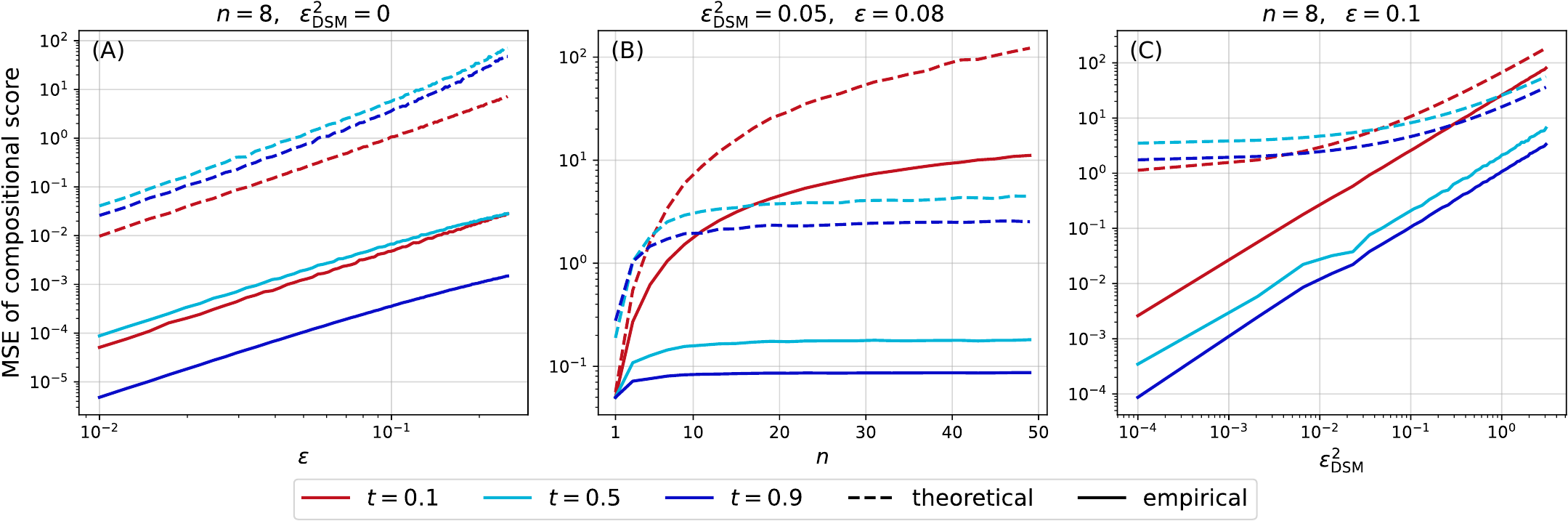}
    \caption{\footnotesize Solid lines stand for the evolution of the empirical MSE of the compositional score estimate~(\ref{eq:score-estimate}) computed with the algorithm \texttt{GAUSS} at different times $t\in[0,1]$ of a diffusion process for the 2D Gaussian example (see Appendix~\ref{gaussian_ex}). Dashed lines represent the evolution of the theoretical bound of the aforementioned compositional score error derived in Proposition~\ref{prop:diffusion_discretization_2}. The empirical and theoretical evolutions are represented with respect to (A) the precision estimation error $\epsilon$ assuming exact individual scores are known, (B) the number of conditional observations $n$ and (C) the individual score error $\epsilon_{\mathrm{DSM}}^2$ that contributes both directly to the compositional score error and indirectly through the precision estimation. The choice of the fixed parameters, especially $\epsilon$ in panel (B) and (C), is discussed in detail in Appendix~\ref{numerical_illustration}.}
    \label{fig:bounds}
\end{figure}
\textbf{Remark.} In the assumptions of Proposition~\ref{prop:diffusion_discretization_2}, we bound the individual score error in expectation over the multi-observation posterior, although one could argue that it would be more natural to bound in expectation over each individual posterior. This choice is questionable, but it ensures that individual score estimates provide good approximations on both the support of the multi-observation posterior, which is our real target, and on the support of each \textit{individual} posterior (see Appendix~\ref{measure_disc} for details). 

\textbf{Numerical illustrations.} We showcase our theoretical bound on a Gaussian example (see Appendix \ref{gaussian_ex} for more details) where all true scores can be analytically computed and for which the Gaussianity assumption of both the backward kernels and the individual posteriors happens to be verified : in this setting, the compositional score (\ref{eq:compositional-score}) corresponds to the true score; therefore, no source of error stemming from external assumptions interferes with the analysis of the compositional MSE. We compare the evolution of our bound to that of the compositional score error obtained empirically, depending on the number of conditional observations $n$, the precision estimation error $\epsilon$ and the individual score error $\epsilon_{\mathrm{DSM}}^2$. When all individual scores are perfectly known ($\epsilon_{\mathrm{DSM}}^2=0$), the sole intrinsic source of error associated with the compositional score~(\ref{eq:score-estimate}) stems from the precision estimation error $\epsilon$, which itself depends only on the discretization hyperparameters of each individual diffusion process and the number of samples used to compute each empirical precision matrix. In such a case, Figure~\ref{fig:bounds}(A) shows that our upper bound successfully captures the global trend of the empirical error. As the number of conditional observations $n$ grows, the empirical compositional score error and our theoretical bound surprisingly do not blow up but rather tend to stabilize more or less quickly depending on the diffusion time: $\Vert \Lambda^{-1}\Vert$ (appearing in both Equation~(\ref{eq:compositional-score}) and our bound) indeed hides a dependence in $1/n$ that seems to counterbalance the simple accumulation of the $n$ individual score errors. Our bound follows this trend especially for large $n$ and/or advanced diffusion times as seen in Figure~\ref{fig:bounds}(B). Finally, Figure~\ref{fig:bounds}(C) shows that our theoretical bound well captures the global evolution of the empirical score error for individual score errors $\epsilon_{\mathrm{DSM}}^2\geq0.01$ but seems to be a bit large for smaller values: the empirical evolution seems to be $\mathcal{O}(\epsilon_{\mathrm{DSM}}^2)$ while our theoretical bound contains a constant term $C$ and is thus in $\mathcal{O}(\epsilon_{\mathrm{DSM}}^2+\epsilon_{\mathrm{DSM}}+C)$, which biases the evolution.

\textbf{Conclusion and perspectives.}
We provide a bound on the mean squared error between the compositional score~(\ref{eq:compositional-score}) and its estimate~(\ref{eq:score-estimate}) as a function of the $n$ individual score errors and the precision estimation error, that is specific to the algorithm. Note that we do not try to quantify the error made by the Gaussian approximations of backward kernels in our analysis, but rather provide a bound for a compositional score estimate computed according to the \texttt{GAUSS} method. Our upper bound could be used to derive convergence bounds for diffusion process using such compositional estimate or even tune the corresponding diffusion discretization hyperparatemeters to achieve a better sampling quality from the multi-observation posterior. Also, our analysis allows to better understand the effect of each source of errors involved in the compositional estimate~(\ref{eq:score-estimate}) as well as the link between the precision estimation error and the individual score error.

It would be interesting to extend our error analysis to other methods for estimating precision matrices, in particular those keeping a dependence in $\btheta$ for such matrices (e.g.  \texttt{JAC} algorithm, see \citealp{linhart2024diffusion}): indeed, the algorithm \texttt{GAUSS} conveniently renders the precision matrices independent of $\btheta$ but at the cost of a strong Gaussianity assumption on each individual posterior. Note that our upper bound remains valid for these other types of estimation methods, if we consider bounding $\sup_{\btheta}\Vert\Sigma_{t,j}(\btheta)-\tilde \Sigma_{t,j}(\btheta)\Vert_{2}$  and not $\mathbb{E}_{\btheta\sim p_t(\btheta\mid \bx_{1:n})}\Vert\Sigma_{t,j}(\btheta)-\tilde \Sigma_{t,j}(\btheta)\Vert_{2}$ as it is done in Proposition~\ref{prop:diffusion_discretization_2} for individual score errors.

\section*{Aknowledgements}
PLCR was supported by a national grant managed by the French National Research Agency (Agence Nationale de la Recherche) attributed to the SBI4C project of the MIAI AI Cluster, under the reference ANR-23-IACL-0006.

\bibliographystyle{apalike}
\bibliography{biblio}


\newpage

\appendix

\raggedbottom

\section{Technical Appendices and Supplementary Material}

\subsection{Conditional score-based modelling}\label{Appendix_A1} 
We consider a stochastic model $\mathcal{M}$, encoded in a simulator, that outputs observations $\bx$ depending on some parameters $\btheta$. We encode knowledge about the parameter space through a prior distribution $\lambda(\btheta)$ and consider the case where the likelihood $p(\bx\mid \btheta)$ of $\mathcal{M}$ is intractable. Given some specific observation $\bx_0$, the goal of Bayesian inference is to sample from the posterior distribution $p(\btheta \mid \bx_0)$ relating the parameters to this specific observation. One way to avoid likelihood evaluations while obtaining the desired samples is to run a (conditional) diffusion process, which requires learning the score of all noisy versions of the posterior distribution $\nabla_{\btheta}\log p_t(\btheta\mid \bx_0)$. As posterior samples are not directly accessible, conditional score-based modeling usually trains a conditional score estimate $s_{\phi}(\btheta,\bx,t)$ by minimizing the following loss in $\phi$ on the \textbf{joint} model $p(\btheta,\bx)=\lambda(\btheta)p(\bx\mid\btheta)$: 
\begin{equation*}
    \mathcal{L}(\phi)=\sum_{t=1}^T\gamma_t^2\mathbb{E}_{\btheta\sim \lambda(\btheta)}\mathbb{E}_{\bx \sim p(\bx\mid \btheta)}\mathbb{E}_{\btheta_t\sim q_{t|0}(\btheta_t|\btheta)}\big[\Vert s_{\phi}(\btheta_t,\bx,t)-\nabla \log q_{t|0}(\btheta_t\mid\btheta)\Vert ^2\big]
\end{equation*} where $\gamma_t$ are positive weights and $q_{t|0}$ denotes a forward diffusion kernel. This loss is the traditional denoising score matching loss for unconditional distributions \citep{sohl2015deep} averaged over the observation space. With this averaging operation it becomes easy to create a training dataset $(\btheta_i,\bx_i)\sim p(\btheta,\bx)$ and use a Monte-Carlo approximation of $\mathcal{L}$. Also, the final score estimate $s_{\phi}$ becomes a valid score approximation for individual posteriors conditioned by any $\bx$, i.e.:
\begin{equation*}
    s_{\phi}(\btheta,\bx,t)\approx \nabla_{\btheta}\log p_t(\btheta\mid \bx) \quad \forall \bx\sim p(\bx), \forall t \in [0,T], \forall \btheta\sim p_t(\btheta\mid \bx).
\end{equation*} If we need to estimate the score of $p_t(\btheta\mid \bx_1)$ where $\bx_1$ is a new observation different from $\bx_0$, it is sufficient to evaluate our estimate at $\bx_1$, i.e. $s_{\phi}(\btheta,\bx_1,t)$ without having to train another score estimate from scratch.
\subsection{Gaussian test case}
\subsubsection{Theoretical setting}\label{gaussian_ex}
We consider a Gaussian prior defined as $\lambda(\btheta)=\mathcal{N}(\btheta;\mu_{\lambda},\Sigma_{\lambda})$ and a Gaussian simulator model (or likelihood) defined as $p(\bx\mid \btheta)=\mathcal{N}(\bx;\btheta,\Sigma)$. The goal of Bayesian inference is to determine the mean $\btheta\in \mathbb{R}^d$ of the Gaussian simulator model, given some specific observation $\bx_0$. This boils down to sampling from the posterior distribution that is also Gaussian 
\begin{equation*}
    p(\btheta\mid \bx_0)=\mathcal{N}(\btheta;\mu_\mathrm{post}(\bx_0),\Sigma_\mathrm{post}),
\end{equation*} where
$\Sigma_\mathrm{post}=(\Sigma^{-1}+\Sigma^{-1}_{\lambda})^{-1}$ and $\mu_\mathrm{post}(\bx_0)=\Sigma_\mathrm{post}(\Sigma^{-1}\bx_0+\Sigma_{\lambda}^{-1}\mu_{\lambda})$.\\
When we have multiple IID observations  $\bx_1,\ldots,\bx_n$ (generated by the same $\btheta$), the multi-observation posterior remains Gaussian:
\begin{equation*}
    p(\btheta\mid \bx_{1:n})=\mathcal{N}(\btheta;\mu_\mathrm{post}(\bx_{1:n}),\Sigma_{\mathrm{post}}(n)),
\end{equation*} where $\Sigma_{\mathrm{post}}(n)=(n\Sigma^{-1}+\Sigma^{-1}_{\lambda})^{-1}$ and $\mu_\mathrm{post}(\bx_{1:n})=\Sigma_{\mathrm{post}}(n)\big(\sum_{i=1}^n\Sigma^{-1}\bx_i+\Sigma^{-1}_{\lambda}\mu_{\lambda}\big)$.
If we follow a Variance-Preserving diffusion process, then the forward transition kernels read as follows:
\begin{equation*}
    q_{t|0}(\btheta_t\mid \btheta_0)=\mathcal{N}(\btheta_t;\sqrt{\alpha_t}\btheta_0,(1-\alpha_t)I) \quad \forall t\in[0,T].
\end{equation*}
We can then analytically find the expressions of the individual posterior scores and prior scores over time as well as the true multi-observation posterior scores (see Appendix D in \cite{linhart2024diffusion}):
\begin{align*}
    \nabla_{\btheta}\log p_t(\btheta\mid \bx_i)&=-\big(\alpha_t\Sigma_\mathrm{post}+(1-\alpha_t)I\big)^{-1}(\btheta-\sqrt{\alpha_t}\mu_\mathrm{post}(\bx_i)),\\
    \nabla_{\btheta}\log \lambda_t(\btheta)&=-\big(\alpha_t\Sigma_{\lambda}+(1-\alpha_t)I\big)^{-1}(\btheta-\sqrt{\alpha_t}\mu_{\lambda}),\\
     \nabla_{\btheta}\log p_t(\btheta\mid \bx_{1:n})&=-\big(\alpha_t\Sigma_\mathrm{post}(n)+(1-\alpha_t)I\big)^{-1}(\btheta-\sqrt{\alpha_t}\mu_\mathrm{post}(\bx_{1:n})).
\end{align*}
\subsubsection{Numerical illustration}\label{numerical_illustration}
We propose a numerical illustration that aims at testing the quality of our upper bound with respect to one source of error, leaving the others fixed. We use the theoretical setting described in Appendix~\ref{gaussian_ex} where the dimension $d=2$.

In panel (A) of Figure~\ref{fig:bounds}, we fix the number of conditional observations $n$ and assume exact individual scores ($\epsilon_{\mathrm{DSM}}^2=0$) : the sole source of error comes from the precision estimation error $\epsilon$. In practice, it can be impacted by the choice of discretization parameters during the backward process and the number of samples used to compute each empirical precision matrix. 

In panel (B), we fix $\epsilon_{\mathrm{DSM}}^2$ and let $n$ grows from $1$ to $50$. The precision error $\epsilon$ used in the bound is fixed but chosen according to $\epsilon_{\mathrm{DSM}}^2$ : we empirically compute the precision error of the $50$ diffusion processes (one per conditional observation) and take the largest value.

In panel (C), we fix $n$ and $\epsilon$ and study the evolution with respect to $\epsilon_{\mathrm{DSM}}^2$. Note that $\epsilon$ used in the theoretical bound corresponds to the largest precision error across all diffusion processes (each of them is specific to a conditional observation and uses a score estimate with a specific error level $\epsilon_{\mathrm{DSM}}^2$).

In Figure~\ref{fig:bounds}, we do not try to assess the quality of the precision bound given in Proposition~\ref{prop:diffusion_discretization_1} but rather the quality of the compositional score error given in Proposition~\ref{prop:diffusion_discretization_2} : in panel (A), we let $\epsilon$ runs over a specific range of values (that can be empirically encountered) and in the last two panels, we set the score errors (fixed or not) and make a fixed choice of discretization parameters, which led to Wasserstein errors during the individual diffusion processes slightly too high, thus not satisfying the assumptions of Proposition~\ref{prop:diffusion_discretization_1}. 
\stepcounter{mainsec}
\subsection{Proof of Proposition~\ref{prop:diffusion_discretization_1}}
\subsubsection{Intermediate results}
We first begin with some lemmas that will be useful to derive the proof of Proposition~\ref{prop:diffusion_discretization_1}. The following lemma as well as the corresponding proof is inspired from Theorem 4.1 in \cite{wass_cov_bound}.
\begin{lemma}\label{lemma_101}
    Let $p$ and $\tilde p$ be two probability distributions with corresponding covariance matrices $\Sigma$, respectively $\tilde \Sigma$. Suppose that $\mathcal{W}_2(p,\tilde p)\leq \eta$ with $0<\eta<\sqrt{\frac{\Vert\Sigma\Vert}{2}}$. 
    Then it holds
    \begin{equation*}
        \Vert\Sigma-\tilde\Sigma\Vert \leq\left(2\sqrt{2}\eta\Vert\Sigma\Vert^{\frac 1 2}+\eta^2(\sqrt{2}+1)\right)\left(\frac{\Vert\Sigma\Vert^{\frac 1 2 }}{\Vert\Sigma\Vert^{\frac 1 2 }-\sqrt{2}\eta}\right).
        \end{equation*}
\end{lemma}

\begin{proof}
    We start by assuming the latent dimension to be one, so that parameters are scalars. 
Suppose that $\btheta\sim p$ and $\tilde \btheta \sim \tilde p$ are distributed according to the optimal coupling for the 2-Wasserstein distance i.e. $\mathcal{W}_2^2( p, \tilde p)=\mathbb{E}\vert\btheta-\tilde\btheta\vert^2$. Let $\mu$ (resp. $\tilde\mu$) be the mean of $ p$ (resp. $\tilde p)$ and $\sigma$ (resp. $\tilde \sigma$) be the standard deviation of $ p$ (resp. $\tilde p$). We assume $\mu=0$ without loss of generality (otherwise, consider the variable $\btheta-\mu$) and $\mathcal{W}_2( p,\tilde p)\leq \eta$. Let $\zeta^2=\mathbb{E}(\btheta^2)=\sigma^2$ and $\tilde \zeta^2=\mathbb{E}(\tilde \btheta^2)=\tilde \sigma^2+\tilde\mu^2$.
\begin{align*}
    |\zeta^2-\tilde\zeta^2|&=|\mathbb{E}(\btheta^2-\tilde \btheta^2)| = |\mathbb{E}[(\btheta-\tilde \btheta)(\btheta+\tilde \btheta)]|\\
    &\leq \mathbb{E}[(\btheta-\tilde \btheta)^2]^{\frac 1 2}\mathbb{E}[(\btheta+\tilde \btheta)^2]^{\frac 1 2} \quad \text{by Cauchy--Schwarz}\\
    &\leq \eta \mathbb{E}[(\btheta+\tilde \btheta)^2]^{\frac 1 2} \quad \text{by optimal coupling and Wasserstein bound}\\
    &\leq \eta \mathbb{E}[2(\btheta^2+\tilde \btheta^2)]^{\frac 1 2} \quad \text{since} \ (a+b)^2\leq 2a^2+2b^2\\
    &= 2^{\frac 1 2}\eta (\mathbb{E}[\btheta^2]+\mathbb{E}[\tilde \btheta^2])^{\frac 1 2}\\
    &\leq 2^{\frac 1 2}\eta \big({\mathbb{E}[\btheta^2]}^{\frac 1 2}+{\mathbb{E}[\tilde \btheta^2]}^{\frac 1 2}\big) \quad \mathrm{since} \ \sqrt{a+b}\leq \sqrt{a}+\sqrt{b} \ \mathrm{for} \ a,b\geq0\\
    &=2^{\frac 1 2}\eta \big(\zeta+\tilde\zeta\big).
\end{align*}
From \citet{wass_cov_bound}, it is also shown that $|\mu-\tilde\mu|=|\tilde\mu|\leq \eta$. Then, 
\begin{align*}
    |\sigma^2-\tilde\sigma^2|&=|\zeta^2-\tilde\zeta^2+\tilde\mu^2|\leq |\zeta^2-\tilde\zeta^2|+|\tilde\mu^2|\\&\leq \sqrt{2}\eta \left(\zeta+\tilde\zeta\right)+\eta^2=\sqrt{2}\eta\left(\sigma+\sqrt{\tilde\sigma^2+\tilde\mu^2}\right)+\eta^2\\
    &\leq \sqrt{2}\eta\left(\sigma+\tilde\sigma+|\tilde\mu|\right)+\eta^2 \quad \mathrm{since} \ \sqrt{a+b}\leq \sqrt{a}+\sqrt{b} \ \mathrm{for} \ a,b\geq0\\
    &\leq \sqrt{2}\eta(\sigma+\tilde\sigma+\eta)+\eta^2 \\
    &= \eta^2(\sqrt{2}+1)+2^{\frac 3 2}\eta\max(\sigma,\tilde\sigma).
\end{align*}
When the dimension is greater than one, we can use Lemma C.$3$ and Corollary C.$2$ of \citet{wass_cov_bound} to extend the previous result and get:
\begin{equation*}
    \Vert\Sigma-\tilde\Sigma\Vert\leq 2^{\frac 3 2}\eta\max(\Vert\Sigma\Vert^{\frac 1 2},\Vert\tilde \Sigma\Vert^{\frac 1 2}) +\eta^2(\sqrt{2}+1).
\end{equation*}
We can pursue the computation to get rid of $\Vert\tilde \Sigma\Vert$. Indeed by triangular inequality we have:
\begin{align*}
    \Vert \tilde \Sigma\Vert^{\frac 1 2 }&\leq \left(\Vert \Sigma\Vert +\Vert \tilde \Sigma-\Sigma\Vert\right)^{\frac 1 2 }\\
    &\leq \Vert \Sigma\Vert^{\frac 1 2 } +\frac{\Vert \tilde \Sigma-\Sigma\Vert}{2\Vert\Sigma\Vert^{\frac 1 2 }} \quad \mathrm{using} \quad \sqrt{a+b}\leq \sqrt{a}+\frac{b}{2\sqrt{a}} \quad \text{for} \quad a\geq0, \ b\geq-a.
\end{align*} So we get the following bound:
\begin{align*}
    &\Vert\Sigma-\tilde\Sigma\Vert\leq 2^{\frac 3 2}\left(\Vert\Sigma\Vert^{\frac 1 2 } +\frac{\Vert \tilde \Sigma-\Sigma\Vert}{2\Vert\Sigma\Vert^{\frac 1 2 }}\right)\eta +\eta^2(\sqrt{2}+1)\\
    &\Rightarrow \left(1-\frac{\sqrt{2}}{\Vert\Sigma\Vert^{\frac 1 2 }}\eta\right)\Vert\Sigma-\tilde\Sigma\Vert \leq2\sqrt{2}\eta\Vert\Sigma\Vert^{\frac 1 2}+\eta^2(\sqrt{2}+1)\\
    &\Rightarrow \Vert\Sigma-\tilde\Sigma\Vert \leq\left(2\sqrt{2}\eta\Vert\Sigma\Vert^{\frac 1 2}+\eta^2(\sqrt{2}+1)\right)\left(\frac{\Vert\Sigma\Vert^{\frac 1 2 }}{\Vert\Sigma\Vert^{\frac 1 2 }-\sqrt{2}\eta}\right)
    \quad \mathrm{since} \quad \eta < \sqrt{\frac{\Vert\Sigma\Vert}{2}}.
\end{align*}
\end{proof}

\begin{lemma}\label{lemma_102}
    Suppose that we have $\Vert\tilde\Sigma-\Sigma\Vert\leq \gamma$ with $\gamma<\frac{1}{\Vert\Sigma^{-1}\Vert}$, then it holds
    \begin{equation*}
        \Vert\tilde\Sigma^{-1}-\Sigma^{-1}\Vert\leq \frac{\gamma\Vert\Sigma^{-1}\Vert^2}{1-\gamma\Vert\Sigma^{-1}\Vert}.
    \end{equation*}
\end{lemma}

\begin{proof}
Suppose that $\|\tilde{\Sigma} - \Sigma\| = \|E\| \leq \gamma$ and that $1/\gamma > \|\Sigma^{-1}\|$ then we can write the following:
\begin{align*}
\tilde{\Sigma}^{-1} &= (\Sigma + E)^{-1},\\
&= \Big(\Sigma (I + \Sigma^{-1}E)\Big)^{-1},\\
&= (I + \Sigma^{-1}E)^{-1}\Sigma^{-1}.
\end{align*}
Now we can write
\begin{align*}
\tilde{\Sigma}^{-1} - \Sigma^{-1} &= (I + \Sigma^{-1}E)^{-1}\Sigma^{-1} - \Sigma^{-1} \\
&=\Big((I + \Sigma^{-1}E)^{-1} - I\Big)\Sigma^{-1} \\
&= - (I + \Sigma^{-1}E)^{-1}\Sigma^{-1}E\Sigma^{-1} \quad \mathrm{since} \quad I=(I+\Sigma^{-1}E)^{-1}(I+\Sigma^{-1}E).
\end{align*}
We will use the following lemma:
\begin{lemma}[\citealp{prec_bound}]
For any matrix, $A\in \mathcal{M}_n(\mathrm{C})$, with $\Vert A\Vert<1$. The matrix $(I-A)$ is invertible and 
\begin{equation*}
    \Vert(I-A)^{-1}\Vert\leq \frac{1}{1-\Vert A\Vert}.
\end{equation*}
\end{lemma}
\noindent Since $\Vert\Sigma^{-1}E\Vert\leq \Vert\Sigma^{-1}\Vert\Vert E\Vert\leq \gamma\Vert\Sigma^{-1}\Vert<1$ by assumption, we can write
\begin{equation*}
    \|(I + \Sigma^{-1}E)^{-1}\| \leq \dfrac{1}{1 - \|\Sigma^{-1}E\|}.
\end{equation*}
\noindent As such, we can bound the norms as follows:
\begin{align*}
\|\tilde{\Sigma}^{-1} - \Sigma^{-1}\| &\leq \|(I + \Sigma^{-1}E)^{-1}\|~\|\Sigma^{-1}\|~\|E\|~\|\Sigma^{-1}\|\\
&\leq \dfrac{\gamma \|\Sigma^{-1}\|^2}{1 - \gamma\|\Sigma^{-1}\|}.
\end{align*}
\end{proof}
\subsubsection{Proof of Proposition~\ref{prop:diffusion_discretization_1}}\label{proof_prop_1}
We are now ready to derive a proof of Proposition~\ref{prop:diffusion_discretization_1} that we restate below.

\textbf{Proposition~\ref{prop:diffusion_discretization_1} (Precision matrix error)}
\textit{
\statementprecisionmatrix 
and function $f:\mathbb{R}_+^2\to\mathbb{R}_+$ is defined in the following proof.}

\begin{proof}
To improve readability, we denote $m=\sqrt{\Vert\Sigma\Vert}$.
Thanks to Lemma~\ref{lemma_101} and since $\eta<\frac{m}{\sqrt{2}}$ we know that $\Vert\Sigma-\tilde\Sigma\Vert \leq\left(2\sqrt{2}\eta m+\eta^2(\sqrt{2}+1)\right)\left(\frac{m}{m-\sqrt{2}\eta}\right)$ if $\mathcal{W}_2(p,\tilde p)\leq \eta$. 

Let $g(\eta)=\eta^2(\sqrt{2}+1)+2\sqrt{2}\eta m-\frac{1}{\Vert\Sigma^{-1}\Vert}\left(1-\frac{\sqrt{2}\eta}{m}\right)=\eta^2(\sqrt{2}+1)+\sqrt{2}\eta(2m+\frac{1}{m\Vert\Sigma^{-1}\Vert})-\frac{1}{\Vert\Sigma^{-1}\Vert}$.

Let $\gamma=\left(2\sqrt{2}\eta m+\eta^2(\sqrt{2}+1)\right)\left(\frac{m}{m-\sqrt{2}\eta}\right)>0$. If we want $\gamma< \frac{1}{\Vert\Sigma^{-1}\Vert}$, we need to solve (in $\eta$) the inequality $g(\eta)< 0$. Let $\Delta=2\left(2m+\frac 1 {m\Vert\Sigma^{-1}\Vert}\right)^2+\frac{4(\sqrt{2}+1)}{\Vert\Sigma^{-1}\Vert}$ be the discriminant of $g$. As it is nonnegative, $g$ admits exactly $2$ roots $\eta_{-}$ and $\eta_{+}$ and $g(\eta)< 0$ for $\eta \in (\eta_{-},\eta_{+})$. After simple computations, we get:
\begin{equation*}
    \eta_{\pm} = \frac{-\sqrt{2}(2m+\frac{1}{m\Vert\Sigma^{-1}\Vert})\pm\sqrt{\Delta}}{2(1+\sqrt{2})},
\end{equation*} 
with $\eta_{-}\leq 0 \leq \eta_{+}$. 
As $\eta$ should be positive, it is sufficient to select $\eta < \eta_{+}=\frac{-\sqrt{2}\left(2\sqrt{\Vert\Sigma\Vert}+\frac{1}{\sqrt{\Vert\Sigma\Vert}\Vert\Sigma^{-1}\Vert}\right)+\sqrt{\Delta}}{2(1+\sqrt{2})}:=f(\Vert\Sigma\Vert,\Vert\Sigma^{-1}\Vert)$ to solve the inequality. This ensures that $\gamma<\frac{1}{\Vert\Sigma^{-1}\Vert}$. Therefore, we can use Lemma~\ref{lemma_102} to get the last inequality of the statement.
\end{proof}
\stepcounter{mainsec}
\subsection{Time-independence of precision errors}\label{intermediate_results}
We show here that the precision estimation error $\epsilon$ is not time dependent, contrary to all other constants described in Proposition~\ref{prop:diffusion_discretization_2}. The algorithm \texttt{GAUSS} assumes that all individual posteriors are Gaussian of the form $p(\btheta\mid \bx_j)=\mathcal{N}(\mu_j(\btheta),\Sigma_j)$ for all $j=1,\ldots,n$. We consider a Variance-Preserving diffusion process that noises these individual posteriors along a forward path using the following Gaussian kernels $q_{t|0}(\btheta_t\mid\btheta)=\mathcal{N}(\sqrt{\alpha_t}\btheta,(1-\alpha_t)I)$. Then known results (e.g. Equation 2.115
in \citealp{bishop}) ensure that $p(\btheta\mid\btheta_t,\bx_j)=\mathcal{N}(\mu_{t,j},\Sigma_{t,j})$ where $\Sigma_{t,j}=(\Sigma_{j}^{-1}+\frac{\alpha_t}{1-\alpha_t}I)^{-1}$. Thus,
\begin{align*}
    \Vert\tilde \Sigma_{t,j}^{-1}-\Sigma_{t,j}^{-1}\Vert&=\Vert\tilde \Sigma_{j}^{-1}+\frac{\alpha_t}{1-\alpha_t}I-(\Sigma_{j}^{-1}+\frac{\alpha_t}{1-\alpha_t}I)\Vert\\
    &=\Vert\tilde \Sigma_{j}^{-1}-\Sigma_{j}^{-1}\Vert\\
    &\leq \epsilon \quad \text{by assumption on precision estimation error}.
\end{align*}
This shows that $\Vert\tilde \Sigma_{t,j}^{-1}-\Sigma_{t,j}^{-1}\Vert$ is bounded uniformly in time by the initial precision estimation error $\epsilon$.
\subsection{Proof of Proposition~\ref{prop:diffusion_discretization_2} - detailed version}
\subsubsection{Proof of intermediate results on compositional score error}
\begin{lemma}\label{lemma_201}
    Suppose that $\Vert\tilde \Sigma^{-1}_{t,i} -\Sigma^{-1}_{t,i} \Vert\leq \epsilon$ for $i=1,\ldots,n$ and $\Vert\tilde \Sigma^{-1}_{t,\lambda} -\Sigma^{-1}_{t,\lambda} \Vert\leq \epsilon_{\lambda}$ then it holds 
  \begin{equation*}
      \Vert\Lambda -\tilde \Lambda \Vert\leq (n-1)\epsilon_{\lambda}+n\epsilon.
  \end{equation*}
\end{lemma}
\begin{proof}

Recall that
\begin{align*}
    \Lambda  &= (1 - n)\Sigma^{-1}_{
t,\lambda}  + \sum_{
i=1}^n\Sigma_{t,i}^{-1},
\end{align*} where we omit the dependence of $\Lambda$ in $\btheta$ since no covariance matrix depends explicitly in $\btheta$ with the \texttt{GAUSS} algorithm. Then, 
\begin{align*}
\Vert\Lambda -\tilde \Lambda \Vert&=\Bigg\|(1-n)\left(\Sigma^{-1}_{
t,\lambda} - \tilde \Sigma^{-1}_{
t,\lambda} \right)+ \sum_{i=1}^n \left(\Sigma_{t,i}^{-1} -\tilde \Sigma_{t,i}^{-1} \right)\Bigg\|\\
&\leq (n-1)\Vert\Sigma^{-1}_{
t,\lambda} - \tilde \Sigma^{-1}_{
t,\lambda} \Vert + \sum_{i=1}^n \Vert\Sigma_{t,i}^{-1} -\tilde \Sigma_{t,i}^{-1} \Vert \ \ \ \mathrm{since} \ \ n\geq 1\\
&\leq (n-1)\epsilon_{\lambda}+n\epsilon.
\end{align*}
\end{proof}
\begin{corollary}\label{corollary_1}
    Let $\epsilon>0$ and $\epsilon_{\lambda}>0$ such that $(n-1)\epsilon_{\lambda}+n\epsilon < \frac{1}{\Vert\Lambda^{-1}\Vert}$. Then,
    \begin{equation*}
        \Vert\tilde\Lambda^{-1} -\Lambda^{-1} \Vert\leq \frac{((n-1)\epsilon_{\lambda}+n\epsilon)\Vert\Lambda^{-1} \Vert^2}{1-\left((n-1)\epsilon_{\lambda}+n\epsilon\right)\Vert\Lambda^{-1} \Vert}.
    \end{equation*}
\end{corollary}
\begin{proof}
We only apply Lemma~\ref{lemma_102} with the correct assumptions.
\end{proof}

\noindent In the following, we set $s_{\phi}(\btheta,\bx_j,t):=s_j(\btheta,t)$ and let  $s_{\lambda}(\btheta,t)\approx\nabla_{\btheta}\log \lambda_t(\btheta)$ for more clarity (recall that in this detailed version of Proposition~\ref{prop:diffusion_discretization_2} we assume that prior scores are unknown and need to be approximated). Thus, the linear combination of individual posterior scores involved in the compositional score expression~(\ref{eq:compositional-score}) and its estimate in (\ref{eq:score-estimate}) reads as follows:
\begin{align*}
\Gamma(\btheta,t;\bx_{1:n}) &= (1 - n) \Sigma^{-1}_{
t,\lambda}\nabla\log \lambda_{t}  + \sum_{i=1}^n \Sigma^{-1}_{t,i}\nabla\log p_t(\btheta\mid \bx_i),\\
    \tilde \Gamma(\btheta,t;\bx_{1:n}) &= (1 - n)\tilde \Sigma^{-1}_{
t,\lambda}s_{\lambda}(\btheta,t) + \sum_{i=1}^n \tilde \Sigma^{-1}_{t,i}s_i(\btheta,t).
\end{align*}

\begin{lemma}\label{lemma_202}
    Suppose that 
\begin{itemize}
    \item $\Vert\tilde \Sigma^{-1}_{
t,\lambda}-\Sigma^{-1}_{
t,\lambda}\Vert\leq \epsilon_{\lambda}$ and $\mathbb{E}_{\btheta \sim p_t(\cdot |\bx_{1:n})}(\Vert s_{\lambda}(\btheta,t)-\nabla_{\btheta}\log \lambda_{t}(\btheta)\Vert^2)\leq \epsilon_{\mathrm{DSM},\lambda}^2$,
\item $\Vert\tilde \Sigma^{-1}_{t,j}-\Sigma^{-1}_{t,j}\Vert\leq \epsilon$ and $\mathbb{E}_{\btheta \sim p_t(\cdot |\bx_{1:n})}(\Vert s_j(\btheta,t)-\nabla_{\btheta}\log p_t(\btheta\mid x_j)\Vert^2)\leq \epsilon_{\mathrm{DSM}}^2$,  
\end{itemize} 
for all $j=1,\ldots,n$. Then it holds
    \begin{align*}
         \mathbb{E}_{p_t(\cdot |\bx_{1:n})}(\Vert\tilde \Gamma(\btheta,t;\bx_{1:n}) &-\Gamma(\btheta\,t;\bx_{1:n})\Vert^2)\\
         &\leq \left[(n-1)\left(\epsilon_{\lambda}\sqrt{\mathbb{E}_{p_t(\cdot |\bx_{1:n})}(\Vert s_{\lambda}(\btheta,t)\Vert^2)}+\Vert\Sigma^{-1}_{
t,\lambda}\Vert\epsilon_{\mathrm{DSM},\lambda}\right)\color{white}\sum \color{black}\right.\\
&\left.+\sum_{i=1}^n\left(\epsilon \sqrt{\mathbb{E}_{p_t(\cdot |\bx_{1:n})}(\Vert s_i(\btheta,t)\Vert^2)}+\Vert\Sigma^{-1}_{t,i}\Vert\epsilon_{\mathrm{DSM}}\right)\right]^2.
    \end{align*}
\end{lemma}
\begin{proof}
We first show intermediate results that will be useful for the proof. In the following the expectations are taken over $\btheta$ drawn from the diffused multi-observation posterior $p_t(\btheta\mid \bx_{1:n})$ and we will only use $\mathbb{E}_{\btheta}$.

\underline{Result (R1)}:
\begin{align*}
\mathbb{E}_{\btheta }&\left(\big\Vert(1 - n)\left(\tilde \Sigma^{-1}_{
t,\lambda}s_{\lambda}(\btheta,t)-\Sigma^{-1}_{
t,\lambda}\nabla_{\btheta}\log \lambda_t(\btheta)\right)\big\Vert^2\right)\\
&\leq (1-n)^2\mathbb{E}_{\btheta}\left(\Vert\tilde \Sigma^{-1}_{
t,\lambda}s_{\lambda}(\btheta,t)-\Sigma^{-1}_{
t,\lambda}s_{\lambda}(\btheta,t)\Vert+\Vert\Sigma^{-1}_{
t,\lambda}s_{\lambda}(\btheta,t)-\Sigma^{-1}_{
t,\lambda}\nabla_{\btheta}\log \lambda_t(\btheta)\Vert\right)^2 \\
&\quad \text{triangular inequality}\\
&= (1-n)^2\left(\mathbb{E}_{\btheta}\Vert\tilde \Sigma^{-1}_{
t,\lambda}s_{\lambda}(\btheta,t)-\Sigma^{-1}_{
t,\lambda}s_{\lambda}(\btheta,t)\Vert^2+\mathbb{E}_{\btheta}\Vert\Sigma^{-1}_{
t,\lambda}s_{\lambda}(\btheta,t)-\Sigma^{-1}_{
t,\lambda}\nabla_{\btheta}\log \lambda_t(\btheta)\Vert^2\right)\\
&\ \ \ + 2(1-n)^2\mathbb{E}_{\btheta}\left(\Vert\tilde \Sigma^{-1}_{
t,\lambda}s_{\lambda}(\btheta,t)-\Sigma^{-1}_{
t,\lambda}s_{\lambda}(\btheta,t)\Vert\Vert\Sigma^{-1}_{
t,\lambda}s_{\lambda}(\btheta,t)-\Sigma^{-1}_{
t,\lambda}\nabla_{\btheta}\log \lambda_t(\btheta)\Vert\right) \\
&\quad \text{expand the square}\\
&\leq (1-n)^2\Vert\tilde \Sigma^{-1}_{
t,\lambda}-\Sigma^{-1}_{
t,\lambda}\Vert^2 \mathbb{E}_{\btheta} (\Vert s_{\lambda}(\btheta,t)\Vert^2)+(1-n)^2\Vert\Sigma^{-1}_{
t,\lambda}\Vert^2\mathbb{E}_{\btheta}\Vert s_{\lambda}(\btheta,t)-\nabla_{\btheta}\log \lambda_t(\btheta)\Vert^2 \\
&\ \ \ + 2(1-n)^2\sqrt{\mathbb{E}_{\btheta}(\Vert\tilde \Sigma^{-1}_{
t,\lambda}s_{\lambda}(\btheta,t)-\Sigma^{-1}_{
t,\lambda}s_{\lambda}(\btheta,t)\Vert^2)}\sqrt{\mathbb{E}_{\btheta}(\Vert\Sigma^{-1}_{
t,\lambda}s_{\lambda}(\btheta,t)-\Sigma^{-1}_{
t,\lambda}\nabla_{\btheta}\log \lambda_t(\btheta)\Vert^2)} \\
&\ \ \ \text{norm inequality and Cauchy--Schwarz inequality}\\
&\leq (1-n)^2\epsilon_{\lambda}^2 \mathbb{E}_{\btheta} (\Vert s_{\lambda}(\btheta,t)\Vert^2)+(1-n)^2\Vert\Sigma^{-1}_{
t,\lambda}\Vert^2\epsilon_{\mathrm{DSM},\lambda}^2\\
&\ \ \ +2(1-n)^2\Vert\tilde \Sigma^{-1}_{
t,\lambda}- \Sigma^{-1}_{
t,\lambda}\Vert\sqrt{\mathbb{E}_{\btheta}(\Vert s_{\lambda}(\btheta,t)\Vert^2)}\Vert\Sigma^{-1}_{
t,\lambda}\Vert\sqrt{\mathbb{E}_{\btheta}(\Vert s_{\lambda}(\btheta,t)-\nabla_{\btheta}\log \lambda_t(\btheta)\Vert^2)}\\
&\leq (1-n)^2\epsilon_{\lambda}^2 \mathbb{E}_{\btheta} (\Vert s_{\lambda}(\btheta,t)\Vert^2)+(1-n)^2\Vert\Sigma^{-1}_{
t,\lambda}\Vert^2\epsilon_{\mathrm{DSM},\lambda}^2\\
&\ \ \ +2(1-n)^2\epsilon_{\lambda}\sqrt{\mathbb{E}_{\btheta}(\Vert s_{\lambda}(\btheta,t)\Vert^2)}\Vert\Sigma^{-1}_{
t,\lambda}\Vert\epsilon_{\mathrm{DSM},\lambda}\\
&=(1-n)^2\left(\epsilon_{\lambda}\sqrt{\mathbb{E}_{\btheta}(\Vert s_{\lambda}(\btheta,t)\Vert^2)}+\Vert\Sigma^{-1}_{
t,\lambda}\Vert\epsilon_{\mathrm{DSM},\lambda}\right)^2 \refstepcounter{result}.    
\tag{\result}              
\label{eq:result_1}  
\end{align*}
In the following results, we will first consider $n$ precision errors $\epsilon_1,\ldots,\epsilon_n$ that we will finally bound by the same constant $\epsilon$ at the end of the proof.

\underline{Result (R2)}:
\begin{align*}
    \mathbb{E}_{\btheta }&\left(\Vert\tilde\Sigma^{-1}_{t,i}s_i(\btheta,t)-\Sigma^{-1}_{t,i}\nabla_{\btheta}\log p_t(\btheta\mid \bx_i)\Vert^2\right)\\
    &\leq \mathbb{E}_{\btheta}\left(\Vert\tilde\Sigma^{-1}_{t,i}s_i(\btheta,t)-\Sigma^{-1}_{t,i}s_i(\btheta,t)\Vert+\Vert\Sigma^{-1}_{t,i}s_i(\btheta,t)-\Sigma^{-1}_{t,i}\nabla_{\btheta}\log p_t(\btheta\mid \bx_i)\Vert\right)^2 \\
    &\quad \text{triangular inequality}\\
    &= \mathbb{E}_{\btheta}\Vert\tilde\Sigma^{-1}_{t,i}s_i(\btheta,t)-\Sigma^{-1}_{t,i}s_i(\btheta,t)\Vert^2 + \mathbb{E}_{\btheta}\Vert\Sigma^{-1}_{t,i}s_i(\btheta,t)-\Sigma^{-1}_{t,i}\nabla_{\btheta}\log p_t(\btheta\mid \bx_i)\Vert^2\\
    &\ \ \ +2\mathbb{E}_{\btheta}\left(\Vert\tilde\Sigma^{-1}_{t,i}s_i(\btheta,t)-\Sigma^{-1}_{t,i}s_i(\btheta,t)\Vert\Vert\Sigma^{-1}_{t,i}s_i(\btheta,t)-\Sigma^{-1}_{t,i}\nabla_{\btheta}\log p_t(\btheta\mid \bx_i)\Vert\right)\\
    &\leq \Vert\tilde\Sigma^{-1}_{t,i}-\Sigma^{-1}_{t,i}\Vert^2\mathbb{E}_{\btheta}\Vert s_i(\btheta,t)\Vert^2 + \Vert\Sigma^{-1}_{t,i}\Vert^2\mathbb{E}_{\btheta}\Vert s_i(\btheta,t)-\nabla_{\btheta}\log p_t(\btheta\mid \bx_i)\Vert^2\\
    &\ \ \ +2\sqrt{\mathbb{E}_{\btheta}(\Vert\tilde\Sigma^{-1}_{t,i}s_i(\btheta,t)-\Sigma^{-1}_{t,i}s_i(\btheta,t)\Vert^2)}\sqrt{\mathbb{E}_{\btheta}(\Vert\Sigma^{-1}_{t,i}s_i(\btheta,t)-\Sigma^{-1}_{t,i}\nabla_{\btheta}\log p_t(\btheta\mid \bx_i)\Vert^2)}\\
    &\leq \epsilon_i^2\mathbb{E}_{\btheta}(\Vert s_i(\btheta,t)\Vert^2) + \Vert\Sigma^{-1}_{t,i}\Vert^2\epsilon_{\mathrm{DSM}}^2\\
    &\ \ \ +2\Vert\tilde\Sigma^{-1}_{t,i}-\Sigma^{-1}_{t,i}\Vert\sqrt{\mathbb{E}_{\btheta}(\Vert s_i(\btheta,t)\Vert^2)}\Vert\Sigma^{-1}_{t,i}\Vert\sqrt{\mathbb{E}_{\btheta}(\Vert s_i(\btheta,t)-\nabla_{\btheta}\log p_t(\btheta\mid \bx_i)\Vert^2)}\\
    &\leq \epsilon_i^2\mathbb{E}_{\btheta}(\Vert s_i(\btheta,t)\Vert^2) + \Vert\Sigma^{-1}_{t,i}\Vert^2\epsilon_{\mathrm{DSM}}^2\\
    &\ \ \ +2\epsilon_i\sqrt{\mathbb{E}_{\btheta}(\Vert s_i(\btheta,t)\Vert^2)}\Vert\Sigma^{-1}_{t,i}\Vert\epsilon_{\mathrm{DSM}}\\
    &=\left(\epsilon_i\sqrt{\mathbb{E}_{\btheta}(\Vert s_i(\btheta,t)\Vert^2)}+\Vert\Sigma^{-1}_{t,i}\Vert\epsilon_{\mathrm{DSM}}\right)^2.
    \refstepcounter{result}       
\tag{\result} 
\label{eq:result_2} 
\end{align*}
\underline{Result (R3)}:
\begin{align*}
    \mathbb{E}_{\btheta}&\left(\Vert\sum_{i=1}^n \tilde\Sigma^{-1}_{t,i}s_i(\btheta,t)-\Sigma^{-1}_{t,i}\nabla_{\btheta}\log p_t(\btheta\mid \bx_i)\Vert^2\right)\\
    &=\mathbb{E}_{\btheta}\sum_{i=1}^n \Vert\tilde\Sigma^{-1}_{t,i}s_i(\btheta,t)-\Sigma^{-1}_{t,i}\nabla_{\btheta}\log p_t(\btheta\mid \bx_i)\Vert^2 \ \ \text{expand the squared norm}\\
    &\ \ \ +\mathbb{E}_{\btheta} \sum_{i=1}^n \sum_{j \neq i}\left\langle \tilde\Sigma^{-1}_{t,i}s_i(\btheta,t)-\Sigma^{-1}_{t,i}\nabla_{\btheta}\log p_t(\btheta\mid \bx_i),\tilde\Sigma^{-1}_{t,j}s_j(\btheta,t)-\Sigma^{-1}_{t,j}\nabla_{\btheta}\log p_t(\btheta\mid x_j)\right\rangle\\
    &=\sum_{i=1}^n \mathbb{E}_{\btheta}\Vert\tilde\Sigma^{-1}_{t,i}s_i(\btheta,t)-\Sigma^{-1}_{t,i}\nabla_{\btheta}\log p_t(\btheta\mid \bx_i)\Vert^2 \ \ \ \text{linearity of expectation}\\
    &\ \ \ +\sum_{i=1}^n \sum_{j \neq i}\mathbb{E}_{\btheta} \left\langle \tilde\Sigma^{-1}_{t,i}s_i(\btheta,t)-\Sigma^{-1}_{t,i}\nabla_{\btheta}\log p_t(\btheta\mid \bx_i),\tilde\Sigma^{-1}_{t,j}s_j(\btheta,t)-\Sigma^{-1}_{t,j}\nabla_{\btheta}\log p_t(\btheta\mid x_j)\right\rangle\\
      &\leq\sum_{i=1}^n \left(\epsilon_i\sqrt{\mathbb{E}_{\btheta}(\Vert s_i(\btheta,t)\Vert^2)}+\Vert\Sigma^{-1}_{t,i}\Vert\epsilon_{\mathrm{DSM}}\right)^2 \ \ \ \text{by Result~(\ref{eq:result_2})}\\
      &\ \ \ +\sum_{i=1}^n \sum_{j \neq i}\mathbb{E}_{\btheta} \Vert \tilde\Sigma^{-1}_{t,i}s_i(\btheta,t)-\Sigma^{-1}_{t,i}\nabla_{\btheta}\log p_t(\btheta\mid \bx_i)\Vert\Vert\tilde\Sigma^{-1}_{t,j}s_j(\btheta,t)-\Sigma^{-1}_{t,j}\nabla_{\btheta}\log p_t(\btheta\mid x_j)\Vert \\
      &\quad \text{using Cauchy--Schwarz inequality}\\
      &\leq\sum_{i=1}^n \left(\epsilon_i\sqrt{\mathbb{E}_{\btheta}(\Vert s_i(\btheta,t)\Vert^2)}+\Vert\Sigma^{-1}_{t,i}\Vert\epsilon_{\mathrm{DSM}}\right)^2\\
      &\ \ \ +\sum_{i=1}^n \sum_{j \neq i}\sqrt{\mathbb{E}_{\btheta} (\Vert \tilde\Sigma^{-1}_{t,i}s_i(\btheta,t)-\Sigma^{-1}_{t,i}\nabla_{\btheta}\log p_t(\btheta\mid \bx_i)\Vert^2)}\sqrt{\mathbb{E}_{\btheta}(\Vert\tilde\Sigma^{-1}_{t,j}s_j(\btheta,t)-\Sigma^{-1}_{t,j}\nabla_{\btheta}\log p_t(\btheta\mid x_j)\Vert^2)}\\
      &\ \ \ \text{Cauchy--Schwarz a second time}\\
       &\leq\sum_{i=1}^n \left(\epsilon_i\sqrt{\mathbb{E}_{\btheta}(\Vert s_i(\btheta,t)\Vert^2)}+\Vert\Sigma^{-1}_{t,i}\Vert\epsilon_{\mathrm{DSM}}\right)^2\\
      &\ \ \ +\sum_{i=1}^n \sum_{j \neq i}\left(\epsilon_i\sqrt{\mathbb{E}_{\btheta}(\Vert s_i(\btheta,t)\Vert^2)}+\Vert\Sigma^{-1}_{t,i}\Vert\epsilon_{\mathrm{DSM}}\right)\left(\epsilon_j\sqrt{\mathbb{E}_{\btheta}(\Vert s_j(\btheta,t)\Vert^2)}+\Vert\Sigma^{-1}_{t,j}\Vert\epsilon_{\mathrm{DSM}}\right) \\
      &\quad \text{by Result~(\ref{eq:result_2})}\\
      &=\left(\sum_{i=1}^n \epsilon_i\sqrt{\mathbb{E}_{\btheta}(\Vert s_i(\btheta,t)\Vert^2)}+\Vert\Sigma^{-1}_{t,i}\Vert\epsilon_{\mathrm{DSM}}\right)^2.
      \refstepcounter{result}       
\tag{\result}   
\label{eq:result_3} 
\end{align*}
We are now ready to derive the proof of Lemma~\ref{lemma_202}.

\begin{align*}
    \mathbb{E}_{\btheta}&\left(\Vert\tilde \Gamma(\btheta,t;\bx_{1:n}) -\Gamma(\btheta,t;\bx_{1:n})\Vert^2\right)\\
 &= \mathbb{E}_{\btheta}\left\Vert(1 - n)\left(\tilde \Sigma^{-1}_{
t,\lambda}s_{\lambda}(\btheta,t)-\Sigma^{-1}_{
t,\lambda}\nabla_{\btheta}\log \lambda_t(\btheta)\right) + \sum_{i=1}^n \left(\tilde\Sigma^{-1}_{t,i}s_i(\btheta,t)-\Sigma^{-1}_{t,i}\nabla_{\btheta}\log p_t(\btheta\mid \bx_i)\right)\right\Vert^2\\
&=\mathbb{E}_{\btheta}\left(\left\Vert(1 - n)\left(\tilde \Sigma^{-1}_{
t,\lambda}s_{\lambda}(\btheta,t)-\Sigma^{-1}_{
t,\lambda}\nabla_{\btheta}\log \lambda_t(\btheta)\right)\right\Vert^2\right) \\
&\ \ \ + \mathbb{E}_{\btheta}\left(\left\Vert\sum_{i=1}^n \left(\tilde\Sigma^{-1}_{t,i}s_i(\btheta,t)-\Sigma^{-1}_{t,i}\nabla_{\btheta}\log p_t(\btheta\mid \bx_i)\right)\right\Vert^2\right)\\
&\ \ \ + 2\mathbb{E}_{\btheta}\left\langle (1 - n)\left(\tilde \Sigma^{-1}_{
t,\lambda}s_{\lambda}(\btheta,t)-\Sigma^{-1}_{
t,\lambda}\nabla_{\btheta}\log \lambda_t(\btheta)\right),\sum_{i=1}^n \left(\tilde\Sigma^{-1}_{t,i}s_i(\btheta,t)-\Sigma^{-1}_{t,i}\nabla_{\btheta}\log p_t(\btheta\mid \bx_i)\right)\right \rangle\\
&\ \ \ \text{expanding the squared norm}\\
&\leq(1-n)^2\left(\epsilon_{\lambda}\sqrt{\mathbb{E}_{\btheta}(\Vert s_{\lambda}(\btheta,t)\Vert^2)}+\Vert\Sigma^{-1}_{
t,\lambda}\Vert\epsilon_{\mathrm{DSM},\lambda}\right)^2 \\
&\quad + \left(\sum_{i=1}^n \epsilon_i\sqrt{\mathbb{E}_{\btheta}(\Vert s_i(\btheta,t)\Vert^2)}\Vert+\Vert\Sigma^{-1}_{t,i}\Vert\epsilon_{\mathrm{DSM}}\right)^2 \\
&\ \ \ + 2\sum_{i=1}^n \mathbb{E}_{\btheta}\left\langle (1-n)\left(\tilde \Sigma^{-1}_{
t,\lambda}s_{\lambda}(\btheta,t)-\Sigma^{-1}_{
t,\lambda}\nabla_{\btheta}\log \lambda_t(\btheta)\right),\tilde\Sigma^{-1}_{t,i}s_i(\btheta,t)-\Sigma^{-1}_{t,i}\nabla_{\btheta}\log p_t(\btheta\mid \bx_i)\right \rangle\\
&\ \ \text{by Results~(\ref{eq:result_1}) and (\ref{eq:result_3}) and linearity of expectation and scalar product}\\
&\leq(1-n)^2\left(\epsilon_{\lambda}\sqrt{\mathbb{E}_{\btheta}(\Vert s_{\lambda}(\btheta,t)\Vert^2)}+\Vert\Sigma^{-1}_{
t,\lambda}\Vert\epsilon_{\mathrm{DSM},\lambda}\right)^2 \\
&\ \ \ + \left(\sum_{i=1}^n \epsilon_i\sqrt{\mathbb{E}_{\btheta}(\Vert s_i(\btheta,t)\Vert^2)}\Vert+\Vert\Sigma^{-1}_{t,i}\Vert\epsilon_{\mathrm{DSM}}\right)^2\\
&\ \ \ + 2\sum_{i=1}^n (n-1)\mathbb{E}_{\btheta}\Vert\tilde \Sigma^{-1}_{
t,\lambda}s_{\lambda}(\btheta,t)-\Sigma^{-1}_{
t,\lambda}\nabla_{\btheta}\log \lambda_t(\btheta)\Vert\Vert\tilde\Sigma^{-1}_{t,i}s_i(\btheta,t)-\Sigma^{-1}_{t,i}\nabla_{\btheta}\log p_t(\btheta\mid \bx_i)\Vert\\
&\ \ \text{by Cauchy--Schwarz} \ \ \langle \bx,\mathbf{y}\rangle \leq |\langle \bx,\mathbf{y}\rangle | \leq \Vert \bx\Vert\Vert \mathbf{y}\Vert \ \ \text{and} \ \ n\geq 1\\
&\leq(1-n)^2\left(\epsilon_{\lambda}\sqrt{\mathbb{E}_{\btheta}(\Vert s_{\lambda}(\btheta,t)\Vert^2)}+\Vert\Sigma^{-1}_{
t,\lambda}\Vert\epsilon_{\mathrm{DSM},\lambda}\right)^2 \\
&\ \ \ + \left(\sum_{i=1}^n \epsilon_i\sqrt{\mathbb{E}_{\btheta}(\Vert s_i(\btheta,t)\Vert^2)}\Vert+\Vert\Sigma^{-1}_{t,i}\Vert\epsilon_{\mathrm{DSM}}\right)^2\\
&\ \ \ + 2(n-1)\sum_{i=1}^n \sqrt{\mathbb{E}_{\btheta}\Vert\tilde \Sigma^{-1}_{
t,\lambda}s_{\lambda}(\btheta,t)-\Sigma^{-1}_{
t,\lambda}\nabla_{\btheta}\log \lambda_t(\btheta)\Vert^2}\sqrt{\mathbb{E}_{\btheta}\Vert\tilde\Sigma^{-1}_{t,i}s_i(\btheta,t)-\Sigma^{-1}_{t,i}\nabla_{\btheta}\log p_t(\btheta\mid \bx_i)\Vert^2} \\
&\ \ \text{by Cauchy--Schwarz} \ \ \mathbb{E}(|XY|)\leq \sqrt{\mathbb{E}(X^2)}\sqrt{\mathbb{E}(Y^2)}\\
&\leq(1-n)^2\left(\epsilon_{\lambda}\sqrt{\mathbb{E}_{\btheta}(\Vert s_{\lambda}(\btheta,t)\Vert^2)}+\Vert\Sigma^{-1}_{
t,\lambda}\Vert\epsilon_{\mathrm{DSM},\lambda}\right)^2 \\
&\ \ \ + \left(\sum_{i=1}^n \epsilon_i\sqrt{\mathbb{E}_{\btheta}(\Vert s_i(\btheta,t)\Vert^2)}+\Vert\Sigma^{-1}_{t,i}\Vert\epsilon_{\mathrm{DSM}}\right)^2\\
&\ \ \ + 2(n-1)\sum_{i=1}^n \left(\epsilon_{\lambda}\sqrt{\mathbb{E}_{\btheta}(\Vert s_{\lambda}(\btheta,t)\Vert^2)}+\Vert\Sigma^{-1}_{
t,\lambda}\Vert\epsilon_{\mathrm{DSM},\lambda}\right)\left(\epsilon_i\sqrt{\mathbb{E}_{\btheta}(\Vert s_i(\btheta,t)\Vert^2)}+\Vert\Sigma^{-1}_{t,i}\Vert\epsilon_{\mathrm{DSM}}\right)\\
&\ \ \ \text{by Results~(\ref{eq:result_1}) and (\ref{eq:result_2})}\\
&=\left((n-1)\left(\epsilon_{\lambda}\sqrt{\mathbb{E}_{\btheta}(\Vert s_{\lambda}(\btheta,t)\Vert^2)}+\Vert\Sigma^{-1}_{
t,\lambda}\Vert\epsilon_{\mathrm{DSM},\lambda}\right)+\sum_{i=1}^n\left(\epsilon_i\sqrt{\mathbb{E}_{\btheta}(\Vert s_i(\btheta,t)\Vert^2)}+\Vert\Sigma^{-1}_{t,i}\Vert\epsilon_{\mathrm{DSM}}\right)\right)^2\\
&\leq\left((n-1)\left(\epsilon_{\lambda}\sqrt{\mathbb{E}_{\btheta}(\Vert s_{\lambda}(\btheta,t)\Vert^2)}+\Vert\Sigma^{-1}_{
t,\lambda}\Vert\epsilon_{\mathrm{DSM},\lambda}\right)+\sum_{i=1}^n\left(\epsilon \sqrt{\mathbb{E}_{\btheta}(\Vert s_i(\btheta,t)\Vert^2)}+\Vert\Sigma^{-1}_{t,i}\Vert\epsilon_{\mathrm{DSM}}\right)\right)^2\\
&\ \ \ \text{since we assume that} \ \epsilon_i\leq\epsilon \ \forall \ i.
\end{align*}
\end{proof}
\begin{lemma}\label{lemma_203}
    Suppose that 
\begin{itemize}
    \item $\Vert\tilde \Sigma^{-1}_{
t,\lambda}-\Sigma^{-1}_{
t,\lambda}\Vert\leq \epsilon_{\lambda}$ and for all $j=1,\ldots,n$ $\Vert\tilde \Sigma^{-1}_{t,j}-\Sigma^{-1}_{t,j}\Vert\leq \epsilon_j$
\item $s_{\lambda}(\btheta,t)$ $\big($resp. $s_{j}(\btheta,t)\big)$ is a score estimate of $\nabla_{\btheta}\log \lambda_t(\btheta)$ $\big($resp. $\nabla_{\btheta}\log p_t(\btheta\mid x_j)\big)$
\end{itemize} Then it holds
    \begin{align*}
         \mathbb{E}_{p_t(\cdot |\bx_{1:n})}(\Vert\tilde \Gamma(\btheta,t;\bx_{1:n})\Vert^2)&\leq \left((n-1)(\Vert \Sigma^{-1}_{
t,\lambda}\Vert+\epsilon_{\lambda})\sqrt{\mathbb{E}_{p_t(\cdot |\bx_{1:n}) }(\Vert s_{\lambda}(\btheta,t)\Vert^2)} \color{white}\sum_{i=1}^n(\Vert \Sigma^{-1}_{t,i}\Vert\color{black} \right.\\
&\left. +\sum_{i=1}^n(\Vert \Sigma^{-1}_{t,i}\Vert+\epsilon_i)\sqrt{\mathbb{E}_{p_t(\cdot |\bx_{1:n})}(\Vert s_i(\btheta,t)\Vert^2)}\right)^2.
    \end{align*}
\end{lemma}
\begin{proof} Here again the expectations are taken over $\btheta$ under the diffused multi-observation posterior $p_t(\btheta\mid \bx_{1:n})$.
\begin{align*}
    \mathbb{E}_{\btheta}\left(\Vert\tilde \Gamma(\btheta,t;\bx_{1:n})\Vert^2\right)&=\mathbb{E}_{\btheta}\left(\Vert(1 - n)\tilde \Sigma^{-1}_{
t,\lambda}s_{\lambda}(\btheta,t) + \sum_{i=1}^n \tilde \Sigma^{-1}_{t,i}s_i(\btheta,t)\Vert^2\right)\\
&= (1 - n)^2\mathbb{E}_{\btheta }\left(\Vert\tilde \Sigma^{-1}_{
t,\lambda}s_{\lambda}(\btheta,t)\Vert^2\right)+\sum_{i=1}^n \mathbb{E}_{\btheta }\left(\Vert\tilde \Sigma^{-1}_{t,i}s_i(\btheta,t)\Vert^2\right)\\
&\ \ \ \text{expanding the squared norm}\\
&\ \ \ +2(1-n)\mathbb{E}_{\btheta }\left \langle \tilde \Sigma^{-1}_{
t,\lambda}s_{\lambda}(\btheta,t),\sum_{i=1}^n \tilde \Sigma^{-1}_{t,i}s_i(\btheta,t)\right \rangle\\
&= (1 - n)^2\mathbb{E}_{\btheta }\left(\Vert\tilde \Sigma^{-1}_{
t,\lambda}s_{\lambda}(\btheta,t)\Vert^2\right)+\sum_{i=1}^n \mathbb{E}_{\btheta }\left(\Vert\tilde \Sigma^{-1}_{t,i}s_i(\btheta,t)\Vert^2\right)\\
&\ \ \ +2(1-n)\sum_{i=1}^n\mathbb{E}_{\btheta }\left \langle \tilde \Sigma^{-1}_{
t,\lambda}s_{\lambda}(\btheta,t), \tilde \Sigma^{-1}_{t,i}s_i(\btheta,t)\right \rangle \\
&\ \ \ \text{linearity of scalar product}\\
&\leq (1 - n)^2\Vert\tilde \Sigma^{-1}_{
t,\lambda}\Vert^2\mathbb{E}_{\btheta }\left(\Vert s_{\lambda}(\btheta,t)\Vert^2\right)+\sum_{i=1}^n \Vert\tilde \Sigma^{-1}_{t,i}\Vert^2\mathbb{E}_{\btheta }\left(\Vert s_i(\btheta,t)\Vert^2\right)\\
&\ \ \ +2(n-1)\sum_{i=1}^n\mathbb{E}_{\btheta }\left(\Vert\tilde \Sigma^{-1}_{
t,\lambda}s_{\lambda}(\btheta,t)\Vert\Vert\tilde \Sigma^{-1}_{t,i}s_i(\btheta,t)\Vert\right) \\
&\ \ \ \text{Cauchy--Schwarz and} \ \ n\geq 1\\
&\leq (1 - n)^2\Vert\tilde \Sigma^{-1}_{
t,\lambda}\Vert^2\mathbb{E}_{\btheta }\Vert s_{\lambda}(\btheta,t)\Vert^2+\sum_{i=1}^n \Vert\tilde \Sigma^{-1}_{t,i}\Vert^2\mathbb{E}_{\btheta }\Vert s_i(\btheta,t)\Vert^2\\
&\ \ \ +2(n-1)\sum_{i=1}^n\sqrt{\mathbb{E}_{\btheta }\Vert\tilde \Sigma^{-1}_{
t,\lambda}s_{\lambda}(\btheta,t)\Vert^2}\sqrt{\mathbb{E}_{\btheta }\Vert\tilde \Sigma^{-1}_{t,i}s_i(\btheta,t)\Vert^2} \\
&\ \ \ \text{Cauchy--Schwarz inequality}\\
&\leq (1 - n)^2\Vert\tilde \Sigma^{-1}_{
t,\lambda}\Vert^2\mathbb{E}_{\btheta }\Vert s_{\lambda}(\btheta,t)\Vert^2+\sum_{i=1}^n \Vert\tilde \Sigma^{-1}_{t,i}\Vert^2\mathbb{E}_{\btheta }\Vert s_i(\btheta,t)\Vert^2\\
&\ \ \ +2(n-1)\Vert\tilde \Sigma^{-1}_{
t,\lambda}\Vert\sqrt{\mathbb{E}_{\btheta }(\Vert s_{\lambda}(\btheta,t)\Vert^2)}\sum_{i=1}^n\Vert\tilde \Sigma^{-1}_{t,i}\Vert\sqrt{\mathbb{E}_{\btheta }(\Vert s_i(\btheta,t)\Vert^2)}\\
&\leq (1 - n)^2\Vert\tilde \Sigma^{-1}_{
t,\lambda}\Vert^2\mathbb{E}_{\btheta }\Vert s_{\lambda}(\btheta,t)\Vert^2+\left(\sum_{i=1}^n \Vert\tilde \Sigma^{-1}_{t,i}\Vert\sqrt{\mathbb{E}_{\btheta }\Vert s_i(\btheta,t)\Vert^2}\right)^2\\
&\ \ \ +2(n-1)\Vert\tilde \Sigma^{-1}_{
t,\lambda}\Vert\sqrt{\mathbb{E}_{\btheta }(\Vert s_{\lambda}(\btheta,t)\Vert^2)}\sum_{i=1}^n\Vert\tilde \Sigma^{-1}_{t,i}\Vert\sqrt{\mathbb{E}_{\btheta }(\Vert s_i(\btheta,t)\Vert^2)}\\
&\ \ \text{since} \  \sum a_i^2\leq \left(\sum a_i\right)^2 \ \text{when} \ a_i\geq 0\\
&=\left((n-1)\Vert\tilde \Sigma^{-1}_{
t,\lambda}\Vert\sqrt{\mathbb{E}_{\btheta }(\Vert s_{\lambda}(\btheta,t)\Vert^2)}+\sum_{i=1}^n\Vert\tilde \Sigma^{-1}_{t,i}\Vert\sqrt{\mathbb{E}_{\btheta }(\Vert s_i(\btheta,t)\Vert^2)}\right)^2\\
&\leq \left[(n-1)(\Vert \Sigma^{-1}_{
t,\lambda}\Vert+\epsilon_{\lambda})\sqrt{\mathbb{E}_{\btheta }(\Vert s_{\lambda}(\btheta,t)\Vert^2)}\right.\\
&\left.+\sum_{i=1}^n(\Vert \Sigma^{-1}_{t,i}\Vert+\epsilon_i)\sqrt{\mathbb{E}_{\btheta }(\Vert s_i(\btheta,t)\Vert^2)}\right]^2\\ 
&\text{using triangular inequality to bound} \ \ \Vert\tilde \Sigma^{-1}_{t,\lambda}\Vert \ \ \text{and} \ \ \Vert\tilde \Sigma^{-1}_{t,i}\Vert.
\end{align*}
\end{proof}
\subsubsection{Proof of detailed version of Proposition~\ref{prop:diffusion_discretization_2}}\label{proof_prop_2}
The final expressions of the multi-observation score~(\ref{eq:compositional-score}) and its estimate~(\ref{eq:score-estimate}) are the following:
\begin{align*}
    \nabla\log p_t(\btheta\mid \bx_{1:n}) &= \Lambda ^{-1}\Gamma(\btheta,t;\bx_{1:n}),\\
    s(\btheta,t;\bx_{1:n})&=\tilde \Lambda ^{-1}\tilde \Gamma(\btheta,t;\bx_{1:n}),
\end{align*}
where $\Lambda = (1-n)\Sigma_{t,\lambda}^{-1} + \sum_{i = 1}^n \Sigma_{t,i}^{-1}$ and $\tilde{\Lambda}= (1-n)\tilde\Sigma_{t,\lambda}^{-1} + \sum_{i = 1}^n \tilde\Sigma_{t,i}^{-1}$.\\
We propose here a more detailed version of Proposition~\ref{prop:diffusion_discretization_2} than what was presented in the main text: we do not assume a Gaussian prior, and thus need to estimate both the prior score $\nabla_{\btheta}\log \lambda_t(\btheta)$ and the precision matrix $\Sigma_{\lambda}^{-1}$.\\

\textbf{Proposition~\ref{prop:diffusion_discretization_2} (Compositional score error - detailed version)}
\textit{Let  $\eta$ (resp. $\eta_{\lambda}$) be chosen as in Proposition~\ref{prop:diffusion_discretization_1} and denote 
\begin{equation*}
\begin{array}{rcl}
M = \max_j \|\Sigma_{t, j}^{-1}\| &\text{with}& L=\max_j \mathbb{E}_{\btheta \sim p_t(\cdot \mid \bx_{1:n})}\left(\Vert\nabla\log p_t(\btheta| \bx_j)\Vert^2\right) \\[1em]
M_\lambda \geq \|\Sigma^{-1}_{t, \lambda}\| & \text{with} & L_{\lambda} \geq \mathbb{E}_{\btheta \sim p_t(\cdot \mid \bx_{1:n})}\left(\Vert\nabla\log \lambda_t(\btheta)\Vert^2\right).
\end{array}
\end{equation*}
Suppose that
\begin{align*}
    \mathbb{E}_{\btheta \sim p_t(\cdot |\bx_{1:n})}(\Vert \nabla_{\btheta}\log p_t(\btheta|\bx_j)-s_{\phi}(\btheta,\bx_j,t)\Vert^2) &\leq \epsilon_{\mathrm{DSM}}^2, \ \text{for all}\quad j=1,\ldots,n\\
    \mathbb{E}_{\btheta \sim p_t(\cdot |\bx_{1:n})}(\Vert \nabla_{\btheta}\log \lambda_t(\btheta)-s_{\lambda}(\btheta,t)\Vert^2) &\leq \epsilon_{\mathrm{DSM},\lambda}^2\\
    \Vert\tilde\Sigma_j^{-1}-\Sigma_j^{-1}\Vert &\leq \epsilon, \ \text{for all}\quad j=1,\ldots,n\\ 
    \Vert\tilde\Sigma_{\lambda}^{-1}-\Sigma_{\lambda}^{-1}\Vert &\leq \epsilon_{\lambda}, 
\end{align*}
with $(n-1)\epsilon_{\lambda}+n \epsilon < \frac{1}{\Vert\Lambda^{-1}\Vert}$. Then it holds 
\begin{multline*}
         \mathbb{E}_{\btheta \sim p_t(\cdot \mid \bx_{1:n})}\big[\Vert\nabla\log p_t(\btheta\mid \bx_{1:n})-s(\btheta,\bx_{1:n},t)\Vert^2 \big]\\
         \leq 
         \left[(n-1)\Vert\Lambda ^{-1}\Vert(\sqrt{L_{\lambda}}+\epsilon_{\mathrm{DSM},\lambda})\left(\epsilon_{\lambda}+\frac{((n-1)\epsilon_{\lambda}+n\epsilon)\Vert\Lambda^{-1} \Vert(M_{\lambda}+\epsilon_{\lambda})}{1-((n-1)\epsilon_{\lambda}+n\epsilon)\Vert\Lambda^{-1} \Vert}\right)\right.\\
         \left.
         + n \Vert\Lambda ^{-1}\Vert(\sqrt{L}+\epsilon_{\mathrm{DSM}})\left(\epsilon+\frac{((n-1)\epsilon_{\lambda}+n\epsilon)\Vert\Lambda^{-1} \Vert(M+\epsilon)}{1-((n-1)\epsilon_{\lambda}+n\epsilon)\Vert\Lambda^{-1} \Vert} \right)\right.\\
         \left. \color{white} \frac{M}{\Vert\Lambda\Vert}\color{black}+\Vert\Lambda^{-1}\Vert ((n-1)M_{\lambda}\epsilon_{\mathrm{DSM},\lambda}+n M\epsilon_{\mathrm{DSM}}\right]^2.
\end{multline*}
If $\epsilon_{\mathrm{DSM}}$ (resp. $\epsilon_{\mathrm{DSM},\lambda}$) is sufficiently small and accompanied by a proper choice of diffusion hyperparameters such that $\mathcal{W}_2(\tilde p(\btheta\mid \bx_j),p(\btheta\mid \bx_j))\leq \eta$ for all $j=1,\ldots,n$ (resp. $\mathcal{W}_2(\tilde \lambda(\btheta),\lambda(\btheta))\leq \eta_{\lambda}$) (Proposition 4 in~\citealp{gao2025wasserstein}), then the precision estimation error $\epsilon$ (resp. $\epsilon_{\lambda}$ for the prior precision error) can be further bounded using Proposition~\ref{prop:diffusion_discretization_1} and $\eta$ (resp. $\eta_{\lambda}$)}.

\begin{proof}
To simplify notations in the proofs, we write the expectation taken over the diffused multi-observation posterior $\mathbb{E}_{\btheta \sim p_t(\cdot \mid \bx_{1:n})}$ simply by $\mathbb{E}_{\btheta}$, and we use $s_i(\btheta,t):=s_{\phi}(\btheta,\bx_i,t)$ and $s_{\lambda}(\btheta,t)\approx \nabla_{\btheta}\log \lambda_t(\btheta)$. 

Recall that precision errors are time independent (see Appendix~\ref{intermediate_results}), i.e. for all $t\in [0,T]$:
\begin{equation*}
    \Vert\tilde \Sigma^{-1}_{t,j}-\Sigma^{-1}_{t,j}\Vert \leq \epsilon_j\leq\epsilon \quad \forall j=1,\ldots,n \quad \text{and} \quad \Vert\tilde \Sigma^{-1}_{t,\lambda}-\Sigma^{-1}_{t,\lambda}\Vert \leq \epsilon_{\lambda}.
\end{equation*} 
We have
\begin{align*}
    &\mathbb{E}_{\btheta }\Vert\nabla\log p_t(\btheta\mid \bx_{1:n})-s(\btheta,t;\bx_{1:n})\Vert^2\\
    &=\mathbb{E}_{\btheta}\Vert\Lambda ^{-1}\Gamma(\btheta,t;\bx_{1:n})-\tilde \Lambda ^{-1}\tilde \Gamma(\btheta,t;\bx_{1:n})\Vert^2\\
    &\leq \mathbb{E}_{\btheta}\left(\Vert\Lambda ^{-1}\Gamma(\btheta,t;\bx_{1:n})-\Lambda ^{-1}\tilde \Gamma(\btheta,t;\bx_{1:n})\Vert  + \Vert\Lambda ^{-1}\tilde \Gamma(\btheta,t;\bx_{1:n})-\tilde \Lambda ^{-1}\tilde \Gamma(\btheta,t;\bx_{1:n})\Vert\right)^2 \\ &\quad \text{using triangular inequality}\\
    &=\mathbb{E}_{\btheta}\Vert\Lambda ^{-1}\Gamma(\btheta,t;\bx_{1:n})-\Lambda ^{-1}\tilde \Gamma(\btheta,t;\bx_{1:n})\Vert^2  + \mathbb{E}_{\btheta}\Vert\Lambda ^{-1}\tilde \Gamma(\btheta,t;\bx_{1:n})-\tilde \Lambda ^{-1}\tilde \Gamma(\btheta,t;\bx_{1:n})\Vert^2\\
    &\ \ \ +2\mathbb{E}_{\btheta}\left(\Vert\Lambda ^{-1}\Gamma(\btheta,t;\bx_{1:n})-\Lambda ^{-1}\tilde \Gamma(\btheta,t;\bx_{1:n})\Vert\Vert\Lambda ^{-1}\tilde \Gamma(\btheta,t;\bx_{1:n})-\tilde \Lambda ^{-1}\tilde \Gamma(\btheta,t;\bx_{1:n})\Vert\right)\\
    &\quad \text{expanding the square}\\
    &\leq \mathbb{E}_{\btheta}\Vert\Lambda ^{-1}\Gamma(\btheta,t;\bx_{1:n})-\Lambda ^{-1}\tilde \Gamma(\btheta,t;\bx_{1:n})\Vert^2  + \mathbb{E}_{\btheta}\Vert\Lambda ^{-1}\tilde \Gamma(\btheta,t;\bx_{1:n})-\tilde \Lambda ^{-1}\tilde \Gamma(\btheta,t;\bx_{1:n})\Vert^2\\
    &\ \ \ +2\sqrt{\mathbb{E}_{\btheta}\Vert\Lambda ^{-1}\Gamma(\btheta,t;\bx_{1:n})-\Lambda ^{-1}\tilde \Gamma(\btheta,t;\bx_{1:n})\Vert^2}\sqrt{\mathbb{E}_{\btheta}\Vert\Lambda ^{-1}\tilde \Gamma(\btheta,t;\bx_{1:n})-\tilde \Lambda ^{-1}\tilde \Gamma(\btheta,t;\bx_{1:n})\Vert^2}\\
    &\quad \text{by Cauchy--Schwarz}\\
    &\leq \Vert\Lambda ^{-1}\Vert^2\mathbb{E}_{\btheta}\Vert\Gamma(\btheta,t;\bx_{1:n})-\tilde \Gamma(\btheta,t;\bx_{1:n})\Vert^2  + \Vert\Lambda ^{-1}-\tilde \Lambda ^{-1}\Vert^2\mathbb{E}_{\btheta}\Vert\tilde \Gamma(\btheta,t;\bx_{1:n})\Vert^2\\
    &\ \ \ +2\Vert\Lambda ^{-1}\Vert\sqrt{\mathbb{E}_{\btheta}(\Vert\Gamma(\btheta,t;\bx_{1:n})-\tilde \Gamma(\btheta,t;\bx_{1:n})\Vert^2)}\Vert\Lambda ^{-1}-\tilde \Lambda ^{-1}\Vert\sqrt{\mathbb{E}_{\btheta}(\Vert\tilde \Gamma(\btheta,t;\bx_{1:n})\Vert^2)}\\
    &=\left(\Vert\Lambda ^{-1}\Vert\sqrt{\mathbb{E}_{\btheta}\Vert\Gamma(\btheta,t;\bx_{1:n})-\tilde \Gamma(\btheta,t;\bx_{1:n})\Vert^2}+\Vert\Lambda ^{-1}-\tilde \Lambda ^{-1}\Vert\sqrt{\mathbb{E}_{\btheta}\Vert\tilde \Gamma(\btheta,t;\bx_{1:n})\Vert^2}\right)^2\\
    &\leq\left[\Vert\Lambda ^{-1}\Vert\left((n - 1)\left(\epsilon_{\lambda}\sqrt{\mathbb{E}_{\btheta}(\Vert s_{\lambda}(\btheta,t)\Vert^2)}+\Vert\Sigma^{-1}_{t,\lambda}\Vert\epsilon_{\mathrm{DSM},\lambda}\right)+\sum_{i=1}^n\left(\epsilon \sqrt{\mathbb{E}_{\btheta}(\Vert s_i(\btheta,t)\Vert^2)}\Vert+\Vert\Sigma^{-1}_{t,i}\Vert\epsilon_{\mathrm{DSM}}\right)\right)\right.\\
    & \left. +\Vert\Lambda ^{-1}-\tilde \Lambda ^{-1}\Vert\left((n-1)\left(\Vert\Sigma^{-1}_{
t,\lambda}\Vert+\epsilon_{\lambda}\right)\sqrt{\mathbb{E}_{\btheta}(\Vert s_{\lambda}(\btheta,t)\Vert^2)}+\sum_{i=1}^n\left(\Vert\Sigma^{-1}_{
t,i}\Vert+\epsilon_i\right)\sqrt{\mathbb{E}_{\btheta}(\Vert s_i(\btheta,t)\Vert^2)}\right)\right]^2\\
&\quad \text{using Lemma~\ref{lemma_202} and Lemma~\ref{lemma_203}}\\
&\leq\left[(n-1)\Vert\Lambda ^{-1}\Vert\sqrt{\mathbb{E}_{\btheta}(\Vert s_{\lambda}(\btheta,t)\Vert^2)}\left(\epsilon_{\lambda}+\frac{((n-1)\epsilon_{\lambda}+n\epsilon)\Vert\Lambda^{-1} \Vert}{1-\left((n-1)\epsilon_{\lambda}+n\epsilon\right)\Vert\Lambda^{-1} \Vert}(\Vert\Sigma^{-1}_{
t,\lambda}\Vert+\epsilon_{\lambda} )\right)\right.\\
&\ \ \ +\sum_{i=1}^n \Vert\Lambda ^{-1}\Vert\sqrt{\mathbb{E}_{\btheta}(\Vert s_i(\btheta,t)\Vert^2)}\left(\epsilon+\frac{((n-1)\epsilon_{\lambda}+n\epsilon)\Vert\Lambda ^{-1}\Vert}{1-\left((n-1)\epsilon_{\lambda}+n\epsilon\right)\Vert\Lambda^{-1} \Vert} (\Vert\Sigma^{-1}_{t,i}\Vert+\epsilon)\right)\\
&\left. \ \ \ +\Vert\Lambda ^{-1}\Vert\left((n-1)\Vert\Sigma^{-1}_{t,\lambda}\Vert\epsilon_{\mathrm{DSM},\lambda}+\sum_{i=1}^n\Vert\Sigma^{-1}_{t,i}\Vert\epsilon_{\mathrm{DSM}}\right)\right]^2\\
&\quad \text{using Corollary~\ref{corollary_1}  and assuming} \quad \epsilon_i\leq \epsilon \quad \forall i.
\end{align*}
By assumption, we have $M=\max_i \Vert\Sigma^{-1}_{t,i}\Vert$ and $ \Vert\Sigma^{-1}_{t,\lambda}\Vert \leq M_{\lambda}$. This gives:
\begin{align*}
    \mathbb{E}_{\btheta}&\Vert\nabla\log p_t(\btheta\mid \bx_{1:n})-s(\btheta,t;\bx_{1:n})\Vert^2\\
    &\leq \left[(n-1)\Vert\Lambda ^{-1}\Vert\sqrt{\mathbb{E}_{\btheta}(\Vert s_{\lambda}(\btheta,t)\Vert^2)}\left(\epsilon_{\lambda}+\frac{(n-1)\epsilon_{\lambda}+n\epsilon)\Vert\Lambda^{-1} \Vert}{1-\left((n-1)\epsilon_{\lambda}+n\epsilon\right)\Vert\Lambda^{-1} \Vert} (M_{\lambda}+\epsilon_{\lambda})\right)\right.\\
&\ \ \ +\sum_{i=1}^n \Vert\Lambda ^{-1}\Vert\sqrt{\mathbb{E}_{\btheta}(\Vert s_i(\btheta,t)\Vert^2)}\left(\epsilon+\frac{(n-1)\epsilon_{\lambda}+n\epsilon)\Vert\Lambda^{-1} \Vert}{1-\left((n-1)\epsilon_{\lambda}+n\epsilon\right)\Vert\Lambda^{-1} \Vert} (M+\epsilon)\right)\\
&\left. \ \ \ +\Vert\Lambda ^{-1}\Vert\left((n-1)M_{\lambda}\epsilon_{\mathrm{DSM},\lambda}+nM\epsilon_{\mathrm{DSM}}\right)\color{white}\sqrt{\mathbb{E}} \color{black}\right]^2.
\end{align*}
As $\mathbb{E}_{\btheta}\left(\Vert\nabla\log \lambda_t(\btheta)\Vert^2\right)\leq L_{\lambda}$ and $L=\max_i \mathbb{E}_{\btheta }\Vert\nabla\log p_t(\btheta\mid \bx_i)\Vert^2$, then it holds for all $i=1,\ldots, n$:
\begin{align*}
\mathbb{E}_{\btheta}\Vert s_i(\btheta,t)\Vert^2 &\leq \mathbb{E}_{\btheta}[\left(\Vert\nabla\log p_t(\btheta\mid \bx_i)\Vert+\Vert s_i(\btheta,t)-\nabla\log p_t(\btheta\mid \bx_i)\Vert\right)^2] \\
&\leq \mathbb{E}_{\btheta }\Vert\nabla\log p_t(\btheta\mid \bx_i)\Vert^2 + \mathbb{E}_{\btheta }\Vert s_i(\btheta,t)-\nabla\log p_t(\btheta\mid \bx_i)\Vert^2 \\
&\ \ \ \ + 2\mathbb{E}_{\btheta }\Vert s_i(\btheta,t)-\nabla\log p_t(\btheta\mid \bx_i)\Vert\Vert\nabla\log p_t(\btheta\mid \bx_i)\Vert\\
&\leq \mathbb{E}_{\btheta }\Vert\nabla\log p_t(\btheta\mid \bx_i)\Vert^2 + \mathbb{E}_{\btheta }\Vert s_i(\btheta,t)-\nabla\log p_t(\btheta\mid \bx_i)\Vert^2 \\
&\ \ \ \ + 2\sqrt{\mathbb{E}_{\btheta }\Vert s_i(\btheta,t)-\nabla\log p_t(\btheta\mid \bx_i)\Vert^2}\sqrt{\mathbb{E}_{\btheta }\Vert\nabla\log p_t(\btheta\mid \bx_i)\Vert^2}\\
&\quad \text{by Cauchy--Schwarz inequality}\\
&\leq L + \epsilon_{\mathrm{DSM}}^2 + 2\sqrt{L}\epsilon_{\mathrm{DSM}} \ \ \ \text{by assumption on score matching error}\\
&=(\sqrt{L}+\epsilon_{\mathrm{DSM}})^2.
\end{align*}
Similarly,
\begin{equation*}
    \mathbb{E}_{\btheta}\Vert s_{\lambda}(\btheta,t)\Vert^2 \leq (\sqrt{L_{\lambda}}+\epsilon_{\mathrm{DSM},\lambda})^2.
\end{equation*}
Thus we finally get:
\begin{align*}
    \mathbb{E}_{\btheta \sim p_t(\cdot |\bx_{1:n})}&\Vert\nabla\log p_t(\btheta\mid \bx_{1:n})-s(\btheta,t;\bx_{1:n})\Vert^2\\
    &\leq \left[(n-1)\Vert\Lambda ^{-1}\Vert(\sqrt{L_{\lambda}}+\epsilon_{\mathrm{DSM},\lambda})\left(\epsilon_{\lambda}+\frac{(n-1)\epsilon_{\lambda}+n\epsilon)\Vert\Lambda^{-1} \Vert}{1-\left((n-1)\epsilon_{\lambda}+n\epsilon\right)\Vert\Lambda^{-1} \Vert}(M_{\lambda}+\epsilon_{\lambda})\right)\right.\\
&\ \ \ +n \Vert\Lambda ^{-1}\Vert(\sqrt{L}+\epsilon_{\mathrm{DSM}})\left(\epsilon+\frac{(n-1)\epsilon_{\lambda}+n\epsilon)\Vert\Lambda^{-1} \Vert}{1-\left((n-1)\epsilon_{\lambda}+n\epsilon\right)\Vert\Lambda^{-1} \Vert} (M+\epsilon)\right)\\
&\left. \ \ \ +\Vert\Lambda ^{-1}\Vert\left((n-1)M_{\lambda}\epsilon_{\mathrm{DSM},\lambda}+nM\epsilon_{\mathrm{DSM}}\right)\right]^2\\
&=\Big[(n-1)\Vert\Lambda ^{-1}\Vert\sqrt{L_{\lambda}}\left(\frac{n\epsilon\Vert\Lambda^{-1} \Vert M_{\lambda}}{1-n\epsilon\Vert\Lambda^{-1} \Vert}\right)\\
&\ \ \ +n \Vert\Lambda ^{-1}\Vert(\sqrt{L}+\epsilon_{\mathrm{DSM}})\left(\epsilon+\frac{n\epsilon\Vert\Lambda^{-1} \Vert}{1-n\epsilon\Vert\Lambda^{-1} \Vert} (M+\epsilon)\right)\\
& \ \ \ +\Vert\Lambda ^{-1}\Vert nM\epsilon_{\mathrm{DSM}}\Big]^2 \\
&\quad \text{if we assume} \ \epsilon_{\mathrm{DSM,\lambda}}=\epsilon_{\lambda}=0 \ \text{as in Proposition~\ref{prop:diffusion_discretization_2} in the main section}.
\end{align*}
\end{proof}
\subsection{Discussion on measure choice}\label{measure_disc}
In Proposition~\ref{prop:diffusion_discretization_2}, all the expectations in the assumptions are taken under the \textbf{multi-observation} (diffused) posterior distribution $p_t(\btheta \mid \bx_{1:n})$ even though the variables inside the expectations come from the (diffused) prior $\lambda_t(\btheta)$ or the \textbf{individual} (diffused) posteriors $p_t(\btheta\mid \bx_i)$. This measure choice is particularly visible for the individual score error $\mathbb{E}_{\btheta\sim p_t(\btheta\mid \bx_{1:n})}\Vert\nabla \log p_t(\btheta \mid \bx_i)-s_i(\btheta,t)\Vert^2\leq \epsilon^2_\mathrm{DSM}$, which is usually computed in expectation over the individual posterior $p_t(\btheta\mid \bx_i)$.
But we can link individual posterior to multi-observation posterior as follows:
\begin{align*}
    p_t(\btheta\mid \bx_i)&=\int p(\btheta_0\mid \bx_i)q_{t|0}(\btheta\mid \btheta_0)d\btheta_0 \quad \text{by definition of the individual diffused posterior}\\
    &= \int q_{t|0}(\btheta\mid \btheta_0)\int p(\btheta_0,\bx_{1:i-1},\bx_{i+1:n}\mid \bx_i)dx_{1:i-1}dx_{i+1:n} d\btheta_0 \quad \text{by marginalisation}\\
    &=\int q_{t|0}(\btheta\mid \btheta_0)\int p(\btheta_0\mid \bx_{1:n})p(\bx_{1:i-1},\bx_{i+1:n}\mid \bx_i)dx_{1:i-1}dx_{i+1:n} d\btheta_0\\
    &=\int p(\bx_{1:i-1},\bx_{i+1:n}\mid \bx_i) \int p(\btheta_0\mid \bx_{1:n}) q_{t|0}(\btheta\mid \btheta_0) d\btheta_0 dx_{1:i-1}dx_{i+1:n} \quad \text{by Fubini theorem}\\
    &=\int p(\bx_{1:i-1},\bx_{i+1:n}\mid \bx_i)  p_t(\btheta\mid \bx_{1:n}) dx_{1:i-1}dx_{i+1:n} \quad \text{by definition of the diffused posterior.}
\end{align*}
Thus we can apply the double expectation equality and get for any measurable function $f$:
\begin{equation*}
    \mathbb{E}_{\btheta\sim p_t(\cdot  \mid \bx_i)}(f(\btheta))=\mathbb{E}_{\bx_{1:i-1},\bx_{i+1:n}\sim p(.\mid \bx_i)}\mathbb{E}_{\btheta\sim p_t(\btheta\mid \bx_{1:n})}(f(\btheta)).
\end{equation*}
When $f(\btheta)=\Vert\nabla \log p_t(\btheta \mid \bx_i)-s_i(\btheta,t)\Vert^2$ for any $i=1,\ldots,n$ we have a bound for the individual score error:
\begin{align*}
    \mathbb{E}_{\btheta\sim p_t(\cdot  \mid \bx_i)}\Vert\nabla \log p_t(\btheta \mid \bx_i)-s_i(\btheta,t)\Vert^2&=\mathbb{E}_{\bx_{1:i-1},\bx_{i+1:n}\sim p(.\mid \bx_i)}\mathbb{E}_{\btheta\sim p_t(\btheta\mid \bx_{1:n})}\Vert\nabla \log p_t(\btheta \mid \bx_i)-s_i(\btheta,t)\Vert^2\\
    &\leq \mathbb{E}_{\bx_{1:i-1},\bx_{i+1:n}\sim p(.\mid \bx_i)} (\epsilon_{\mathrm{DSM}}^2) \ \ \ \ \text{by assumption on the individual score error}\\
    &=\epsilon_{\mathrm{DSM}}^2.
\end{align*}
So if we manage to bound the individual score error in average under the \textbf{multi-observation} (diffused) posterior for a specific time $t$ in the diffusion process, then this error will be bounded in average under any \textbf{individual} (diffused) posterior (by the same constant).

\end{document}

%% file: math_commands.tex
\usepackage{amsmath,amsfonts,bm}

\def\eqref#1{equation~\ref{#1}}

\def\1{\bm{1}}

\DeclareMathAlphabet{\mathsfit}{\encodingdefault}{\sfdefault}{m}{sl}
\SetMathAlphabet{\mathsfit}{bold}{\encodingdefault}{\sfdefault}{bx}{n}

